\documentclass{article}
\usepackage{arxiv}

\usepackage[utf8]{inputenc} % allow utf-8 input
\usepackage[T1]{fontenc}    % use 8-bit T1 fonts
\usepackage{hyperref}       % hyperlinks
\usepackage{url}            % simple URL typesetting
\usepackage{booktabs}       % professional-quality tables
\usepackage{amsfonts}       % blackboard math symbols
\usepackage{nicefrac}       % compact symbols for 1/2, etc.
\usepackage{microtype}      % microtypography
\usepackage{lipsum}

\usepackage{mathtools}
\usepackage{latexsym,amsmath,amssymb,bm}
\usepackage{graphicx}
\usepackage[font=footnotesize, skip=5pt]{caption}
\usepackage[font=footnotesize, skip=5pt]{subcaption}
\usepackage{times}
\usepackage{amsthm}
\usepackage[retainorgcmds]{IEEEtrantools}
\usepackage{mdwlist}
\usepackage{ctable}
\usepackage{enumitem}
\usepackage{listings}
\usepackage{mathtools}

\DeclarePairedDelimiter\abs{\lvert}{\rvert}%
\makeatletter
\let\oldabs\abs
\def\abs{\@ifstar{\oldabs}{\oldabs*}}

\newtheorem{theorem}{Theorem}
\newtheorem{defn}{Definition}
\newtheorem{prop}{Proposition}
\newtheorem{lma}{Lemma}

\newtheorem{cor}{Corollary}

\DeclareMathOperator*{\argmin}{arg\,min}

\usepackage{arxiv}
\title{Benign Overfitting and Noisy Features}

\author{
  Zhu Li\\%\thanks{%Use footnote for providing further
  Department of Statistics\\
  University of Oxford\\
  Oxford, OX1 3LB \\
  \texttt{zhu.li@stats.ox.ac.uk} \\
  \And
  Weijie J.~Su \\
  Department of Statistics\\
  The Wharton School, University of Pennsvlvania\\
  Philadelphia, PA 19104 \\
  \texttt{suw@wharton.upenn.edu} \\
  \AND
  Dino Sejdinovic\\
  Department of Statistics\\
  University of Oxford\\
  Oxford, OX1 3LB \\
  \texttt{dino.sejdinovic@stats.ox.ac.uk} \\
}

\begin{document}

\maketitle

\begin{abstract}
Modern machine learning often operates in the regime where the number of parameters is much higher than the number of data points, with zero training loss and yet good generalization, thereby contradicting the classical bias-variance trade-off. This \textit{benign overfitting} phenomenon has recently been characterized using so called \textit{double descent} curves where the risk undergoes another descent (in addition to the classical U-shaped learning curve when the number of parameters is small) as we increase the number of parameters beyond a certain threshold. In this paper, we examine the conditions under which \textit{Benign Overfitting} occurs in the random feature (RF) models, i.e. in a two-layer neural network with fixed first layer weights. We adopt a new view of random feature and show that \textit{benign overfitting} arises due to the noise which resides in such features (the noise may already be present in the data and propagate to the features or it may be added by the user to the features directly) and plays an important implicit regularization role in the phenomenon.

%Modern deep learning often operates in the regime where the number of parameters are much more than the number of data points, yet it has respectable prediction accuracy, contradicting with the classical bias-variance tradeoff. This \textit{Benign Overfitting} phenomenon has recently been characterized as the so called "double descent" curve: while the learning risk follows the classical U-shaped regime when number of parameter is small, it shoots up to infinite when it passes the threshold where it equals to the number of data point. Moreover, the learning risk decreases as we further increase the number of parameters. In this paper, we examine under what conditions will this \textit{Benign Overfitting} happen in the random Fourier feature (RFF) model, i.e. two-layer neural network with fixed first layer weights. We adopt a new view of RFF and hence, find out that the \textit{Benign Overfitting} is due to the noises reside in the features, where the source of the noises either come from data collection procedure or are added manually (e.g. the bias term in neural network can be seen as added noise on to the features). These noises play an important role of implicit regularization which allows us to choose a model with much more parameters while at the same time preventing the risk explodes. 
\end{abstract}

\section{Introduction}
A fundamental task of modern machine learning is to estimate a function from a large (potentially noisy) dataset, where the key is to be able to generalize to new data: given training data $\{(x_i,y_i)\}_{i=1}^n$ drawn from a probability measure $\rho$ defined on $\mathcal{X} \times \mathcal{Y}$, we learn a predictor $f$ such that $f(x)$ is close to the true label $y$ for the previously unseen data point $(x,y)$. The estimating function $f$ is commonly chosen from some hypothesis space $\mathcal{H}$, and typically takes the following form: \[f(x) = \sum_{i=1}^s \beta_i z(x,w_i),\] where $s$ typically grows with $n$ and represents the number of features in the model, $\beta_1,\dots, \beta_s \in \mathbb{R}$ are the coefficients, $z$ is some non-linear function and $w_1,\dots, w_s$ are the parameters associated with $z$. The learning task boils down to estimating $\beta_i$'s and $w_i$'s from noisy training data. A widely adopted estimation procedure is called the \textit{Empirical Risk Minimization} (ERM), where for a loss function $l$, we find the $f \in \mathcal{H}$ with minimum training risk $\frac{1}{n}\sum_{i=1}^n l(y_i,f(x_i))$. We often mitigate the mismatch between minimizing the training risk and minimizing the true risk $\mathbb{E}_{\rho}\left[l(y,f(x))\right]$ (our ultimate objective) by formulating the \textit{Regularized ERM} given by
\begin{IEEEeqnarray}{rCl}
f := \argmin_{f\in \mathcal{H}} \frac{1}{n} \sum_{i=1}^n l(y_i,f(x_i)) + \lambda \Omega(f). \label{eqn:rerm}
\end{IEEEeqnarray}
Here, $\Omega(\cdot)$ is a measure of the function complexity, $\lambda$ is a hyperparameter that controls the complexity of the function and, hence, the capacity of our hypothesis space. Notable examples of Regularized ERM learning include support vector machines \cite{cortes1995support}, random Fourier features \cite{rahimi2007random} and neural networks and deep learning \cite{Goodfellow-et-al-2016}. In all of these examples, the choice of $\lambda$ is informed by the classical learning wisdom that suggests that we should balance between \textit{underfitting} and \textit{overfitting} \cite{geman1992neural,friedman2001elements}:
\begin{itemize}
    \item If $\lambda$ is too large, predictors from $\mathcal{H}$ are likely to be too simple and may \textit{underfit} the data and return both high training and true risk.
    
    \item If $\lambda$ is too small, predictors are too complex such that they are likely to \textit{overfit} the data and return small training risk but high true risk.
\end{itemize}

Therefore, tuning $\lambda$ is of great importance as it trades off the capacity of the hypothesis space $\mathcal{H}$ and the prediction accuracy, which forms the classical U-shape learning curve \cite[Figure 2.11]{friedman2001elements}. \\

However, this classical learning theory has been challenged over recent years. In kernel regression, it has been observed that the lowest prediction error often occurs when $\lambda = 0$ (see e.g. \cite[Figure 1]{belkin2018understand} \& \cite[Figure 1]{liang2018just}). In addition, as demonstrated in Figure 1 and Table 2 in \cite{zhang2016understanding}, state-of-the-art deep networks are often trained to the \emph{interpolation regime} where estimators perfectly fit the training data (i.e. they fit perfectly even the noise present in the labels, which is indicative of overfitting), yet they still generalize well to new examples. Similar behaviour is also observed by interpolating kernel machines and deep networks even in the presence of significant label noise. Finally, in a series of insightful papers \cite{belkin2018understand,belkin2019reconciling,belkin2019two}, it was pointed out that the prediction accuracy for interpolating models empirically often exhibits the so called \emph{double descent} behaviour. In this framework, when the model complexity is small, i.e. $s < n$, we are in the classical U-shape regime. As $s$ approaches $n$, the training risk goes to $0$ while the true risk grows very large. However, as soon as $s$ passes the threshold of $n$, the true risk starts decreasing again. Empirically, it was also observed that the minimum true risk achieved as $s \rightarrow \infty$ is lower than the minimum risk achieved in the $s < n$ regime. \\

This \textit{benign overfitting} phenomenon for the interpolating estimator has drawn much interest in the machine learning community over the last two years. In kernel regression, \cite{liang2018just} derive the learning risk of the interpolating estimator and show that with certain properties of the kernel matrix and training data, there is an implicit regularization coming from the curvature of the kernel function which guarantees a good prediction accuracy. \cite{belkin2019reconciling} experimentally demonstrate the double descent curve in linear and non-linear regression cases, and \cite{belkin2019two} subsequently provide a finite sample analysis of the excess risk for the interpolating estimator in some special settings (where it is assumed that the responses and features are jointly Gaussian). By appealing to random matrix theory, \cite{hastie2019surprises} obtain the asymptotic behaviour of the prediction accuracy in the linear regression setting with correlated features, where sample size $n$ and the covariate dimension $d$ both approach infinity with asymptotic ratio $d/n \rightarrow \gamma\in(0,\infty)$. They also study the asymptotic behaviour of the variance term in the random non-linear feature regression setting. Recently in \cite{mei2019generalization}, by letting $n$, $d$ and $s$ all go to infinity such that $d/n$ and $n/s$ remain bounded, the exact derivation of the double descent curve has been rigorously studied in the random feature regression setting and the asymptotic behaviour is attained. Concurrently, \cite{liao2020random} has further extended the analysis of \cite{mei2019generalization} by relaxing the Gaussian assumptions on the data distribution while all other settings are held the same. A similar asymptotic behaviour of the excess learning risk is obtained and a precise characterization of double descent is demonstrated. Finally, \cite{bartlett2020benign} which is most related to our work, studied the upper and lower bound on the excess risk by assuming that the covariates belong to an infinite-dimensional Hilbert space and follow a sub-exponential distribution. Through investigating the finite sample learning risk behaviour, they explicitly give the conditions for the overfitted linear regression model to have optimal prediction accuracy. Intuitively, the covariance operator spectrum has to decay but slowly enough so that the sum of the tail of its eigenvalues should be large compared to $n$. \\

While many results have been proposed since the benign overfitting was observed, it seems that there is no satisfying answer explaining how this phenomenon happens in general. Also, to the best of our knowledge, current literature seems to overlook an important factor: noise that may exist in the covariates $x$ or in the features $z(x,w)$. Recall that given training data $\{(x_i,y_i)\}_{i=1}^n$, the classical learning assumes the data to be independently identically distributed (i.i.d) samples from joint distribution $\rho$ on which the true risk is also evaluated. The prediction function is to be estimated by applying learning algorithm to the observed $x$ (the covariates) or $z(x,w)$ (the features), where $x$ and $z(x,w)$ are assumed to have \textbf{no noise}. However, in practice, the covariates or the features will often be noisy. The noise may already be present in the data (e.g. arising due to imperfect measurement equipment), or it may be added by the user (as we will elaborate later). In this contribution, we show that such noise in fact acts as an implicit regularizer and gives rise to the benign overfitting phenomenon. We study the effect of the covariates noise under the Random Feature model \cite{rahimi2007random} using features sampled via Gaussian process formalism. The Random Feature model was initially introduced as an effective way of scaling up kernel methods and its approximation properties have been extensively studied in the literature (see \cite{bach2013sharp,bach2017equivalence,avron2017random,rudi2017generalization,li2019towards}). Random Feature model can be regarded as a two-layer neural network with randomly sampled then fixed first layer weights. As a result, this model has attracted much attention as a first step in understanding deep networks \cite{hazan2015steps,de2018gaussian,novak2018bayesian,garriga2018deep,daniely2017sgd}. In our framework, similar to the classical learning, we first transform the covariate $x$ to a feature vector $z(x,w)$ through a non-linear random map $z$ (parametrised by $w$), followed by applying regression to the feature vector $z(x,w)$. However, a key difference in our work is that we assume each feature vector $z(x,w)$ to be corroded with some noise $\xi$. As discussed before, there are two possible sources of $\xi$. The first kind is added deliberately by the user, where a noise term $\xi$ is added to $z(x,w)$. The second could be the noise which already resides in the covariates $x$, i.e. $x_{\xi_0} = x+ \xi_0$, giving rise to a noisy feature. Hence, to simplify the notation, we will thereafter write the noisy feature as $z_{\xi}(x,w) = z(x,w) + \xi$, where $z(x,w)$ is the true feature while $\xi$ represents the noise attached to $z(x,w)$. Regression is then performed on the noisy feature $z_{\xi}(x,w)$. In this noisy feature setting, we are interested in how the presence of $\xi$ will affect the generalization performance of the interpolating estimator. Specifically, we make the following contributions:
\begin{itemize}
    \item By assuming $\xi$ to have a normal distribution, Theorem \ref{theo: over_rff_noisy} establishes a precise relationship between $\xi$ and the excess learning risk. This characterization explains how $\xi$ can act as an implicit regularization in both finite sample and asymptotic cases: allowing us to choose large overparametrized models while preventing the explosion of learning risk;
    
    \item In the setting of regression with Gaussian Process Features (see Section \ref{sec:new_look}), Proposition \ref{theo:over_rff} provides a nearly matching upper and lower bound for excess learning risk of the interpolating estimator and details the conditions for the interpolating estimator to exhibit benign overfitting. Proposition \ref{theo:over_rff} serves as the motivation for us to explore the noisy features;

    \item In Corollary \ref{theo: over_rff_noisy_sug}, we extend our analysis of the excess learning risk to the case where $\xi$ follows a subgaussian distribution. Our results demonstrate that as long as $\xi$ decays with $s$ at a prescribed rate, benign overfitting will occur, whereas the shape of the distribution is not a key component in driving these phenomena;
    
    \item In Corollary \ref{theo:double_descent}, we analyze the behaviour of the excess learning risk bound from Corollary \ref{theo: over_rff_noisy_sug} and explicitly show that our analysis on the excess learning risk leads to the double descent; 
    
    \item Our characterization of the relationship between $\xi$ and the excess learning risk reveals that if $\xi$ is chosen to decay according to a certain rate of $s$, it is possible for the excess learning risk to reduce to $0$. Hence, by choosing a learning machine with potentially much more parameters than the number of data points, it is possible to achieve the optimal risk with a carefully designed decay of $\xi$;

\end{itemize}

All of the above results apply to both finite sample case as well as the asymptotic case, the results are valid for data of arbitrary dimension. In addition, our results only \textbf{impose very weak conditions} on the kernel structure, i.e. as long as its corresponding covariance operator is of trace class (See Assumption A.$1$). This is fundamentally different than the existing analysis for non-linear feature map \cite{mei2019generalization,liao2020random}, which only work for the Gaussian kernel. More importantly, our results have \textbf{no specific assumptions} on the data generation distribution, which is a significant improvement over existing work \cite{bartlett2020benign,mei2019generalization}, as they often assume the Guassian data generation distribution.

\section{Definitions and Notations}
\subsection{Regularised ERM and Kernel Ridge Regression} \label{sec:krr}
Let $\pmb{\mathsf{x}}$ and $\pmb{\mathsf{y}}$ be random variables with joint probability distribution $\rho(x,y)  = \rho_{\pmb{\mathsf{x}}}(x)\rho_{\pmb{\mathsf{y}}}(y|x)$. In this article, we consider the regression problem where  response variable $y$ is real-valued and we use the squared loss $l(y,f(x))=(y-f(x))^2$.

In the regularised ERM with the squared loss, the optimal estimating regression function is given by \[f_*(x) = \mathbb{E}(\pmb{\mathsf{y}}|\pmb{\mathsf{x}} = x).\] Let $X = [x_1, \dots,x_n]^T$ and $Y = [y_1,\cdots,y_n]^T$ denote the training inputs and outputs. Given the function $\hat{f}$ estimated based on $(X,Y)$, we will consider the notion of excess risk as a measure of its generalization performance ~\cite{caponnetto2007optimal}:
\begin{equation}
    R_{X}(\hat{f}) = \mathbb{E}_{\pmb{\mathsf{x}}}\mathbb{E}_{Y|X}[(\hat{f}(\pmb{\mathsf x})- f_*(\pmb{\mathsf x}))^2|X].
\end{equation}
Note that the excess risk is conditional on the training inputs $X$ as emphasized by our notation $R_X$. However, when the context is clear, we will drop the subscript $X$ for brevity.

An important class of regularized ERM problems is kernel ridge regression (KRR) which we describe next.

\begin{defn}(Kernel Ridge Regression (KRR)) \label{def: kernel_regression}
Given training example $\{(x_i,y_i)\}_{i=1}^n$ from $\rho$, a kernel $K$ and its corresponding RKHS $\mathcal{H}$, KRR problem is the regularised ERM in Eq.(\ref{eqn:rerm}) with $l$ being the squared loss and $\Omega(\cdot)$ being the squared RKHS norm:
\begin{IEEEeqnarray}{rCl}
\hat{f}^{\lambda} := \argmin_{f\in\mathcal{H}}\ \frac{1}{n}\sum_{i=1}^n (y_i-f(x_i))^2 + \lambda \|f\|_{\mathcal{H}}^2. \label{krr_opm}
\end{IEEEeqnarray}
\end{defn}
Here, $\lambda$ is the regularization parameter. Applying the representer theorem \cite[Theorem 5.5]{steinwart2008support}, the solution to Eq.(\ref{krr_opm}) can be written as $\hat{f}(x) = K(x,X)(\mathbf{K}+n\lambda I)^{-1} Y$, where $K(x,X) = [K(x,x_1),\dots,K(x,x_n)]^T$ and $\mathbf{K}_{i,j} = K(x_i,x_j)$ is the Gram matrix. 

\subsection{Random Feature Approximation}\label{sec:new_look}
Traditionally, random feature approximation is a simple way to construct a finite-dimensional approximation of an infinite-dimensional kernel introduced by \cite{rahimi2007random}. However, in this paper, we will adopt a new perspective on random feature approximation. Specifically, we consider kernel ridge regression learning with kernel $K$ based on \textit{Gaussian process features}. By the Karhunen-Loeve expansion theorem \cite[Theorem 4.3]{kanagawa2018gaussian}, also see \cite[Lemma 3.3 and 3.7]{steinwart2019convergence}, under suitable condition of $K$, a Gaussian process $f_K \sim \mathcal{GP}(0, K)$ has the expansion as \[f_K(x) = \sum_{i=1}^P \lambda_i^{1/2}e_i(x)w_i,\] where $w_i$'s are i.i.d $\sim \mathcal{N}(0,1)$ and $(\lambda_i, e_i)_{i=1}^P$ are the eigensystem corresponding to Mercer's decomposition \cite[Theorem 4.2]{kanagawa2018gaussian} of $K$: \[K(x,y) = \sum_{i =1}^P \lambda_i e_i(x) e_i(y),\] with $\{e_i\}_{i =1}^P$ being an at most countable orthonormal set of $L_2(d\rho)$ and $\{\lambda_i\}_{i=1}^P$ a sequence of non-increasing strictly positive eigenvalues.\\

In this article, kernel $K$ can be infinite-dimensional, i.e. $P = \infty$. However, in our analysis, we will also define a low-rank kernel $k$ that approximates $K$ by only using the top $p<P$ eigenvalues and eigenvectors of Mercer's decomposition: \[k(x,y) = \sum_{i =1}^p \lambda_i e_i(x) e_i(y).\]  In addition, if we denote $V = [e_1(\cdot),\dots,e_p(\cdot)]^T$, $V(x) = [e_1(x),\dots,e_{p}(x)]^T$ and $D = [\lambda_1,\dots,\lambda_p],$ then we can write \[k(x,y) = V(x)^TDV(y).\]

From now on, we will use kernel $k$ and whenever we need to refer to $K$, we will treat $K$ as the limit of $k$ when $p \rightarrow P$. For GP $ f \sim \mathcal{GP}(0,k)$, we can express it using Karhunen-Loeve expansion \[f = V^TD^{1/2}\pmb{\mathsf{w}}, \] where $\pmb{\mathsf{w}}$ is a $p$-dimensional Gaussian random vector with each entry being standard normal random variable.\\

Now if we sample $\mathbf{w}^{(1)},\dots, \mathbf{w}^{(s)}$ i.i.d $\sim \pmb{\mathsf{w}}$, and let $z(\mathbf{w}^{(i)},\cdot) = V^TD^{1/2} \mathbf{w}^{(i)}$, then $z(\mathbf{w}^{(1)},\cdot), \dots ,z(\mathbf{w}^{(s)},\cdot)$ are i.i.d sample paths $\sim \mathcal{GP}(0,k)$, such that $\mathbb{E}_{\pmb{\mathsf{w}}}(z(\mathbf{w}^{(i)},x))= 0$ and $\mathbb{E}_{\pmb{\mathsf{w}}}(z(\mathbf{w}^{(i)},x)z(\mathbf{w}^{(i)},y)) = k(x,y)$.  In addition, we let
\begin{IEEEeqnarray}{rCl}
\mathbf{z}_x(\mathbf{W}) &=& \frac{1}{\sqrt{s}} \mathbf{W}^TD^{1/2}V(x); \nonumber\\
\mathbf{Z} =  [\mathbf{z}_{x_1}(\mathbf{W}),&\dots,&\mathbf{z}_{x_n}(\mathbf{W})]^T= \frac{1}{\sqrt{s}}V(X)D^{1/2}\mathbf{W};\nonumber
\end{IEEEeqnarray}
where \[\mathbf{W} = [\mathbf{w}^{(1)},\dots, \mathbf{w}^{(s)}] \in \mathbb{R}^{p \times s}, ~~~V(X) = [V(x_1),\dots, V(x_n)]^T.\] It is easy to verify that $k(x,y) = \mathbb{E}_{\pmb{\mathsf{w}}}[\mathbf{z}_{x}(\mathbf{W})^T\mathbf{z}_{y}(\mathbf{W})] $ and $\mathbf{K} = \mathbb{E}_{\pmb{\mathsf{w}}}(\mathbf{Z}\mathbf{Z}^T)$.

\paragraph{Covariance Operator} We define the following various forms of covariance operator:
\begin{itemize}
    \item Let $\Sigma$ be the population covariance operator for kernel $k$ with eigenvalue matrix $D$ \footnote{Note that we use $D$ here because $\Sigma$ has the same eigenvalues as the Mercer's decomposition of kernel $k$.}: \[\Sigma = \mathbb{E}_{\pmb{\mathsf{x}}} \left\{D^{1/2} V(\pmb{\mathsf{x}}) V(\pmb{\mathsf{x}})^T D^{1/2}\right\}.\]
    
    \item Let $\hat{\Sigma}$ be the sample estimate of $\Sigma$ \[\hat{\Sigma} = \frac{1}{n}\sum_{i=1}^n D^{\frac{1}{2}} V(x_i)V(x_i)^T D^{\frac{1}{2}}.\] Let $\hat{D} = \textnormal{diag}(\hat{\lambda}_1,\dots,\hat{\lambda}_n)$ be the eigenvalue matrix of $\hat{\Sigma}$, then asymptotically \[\lim_{n\rightarrow \infty} \hat{\Sigma} = \Sigma, ~~~\lim_{n\rightarrow \infty} \hat{D} = D;\]
    
    \item Let $\hat{\Sigma}^s$ be the random feature approximation of $\hat{\Sigma}$ with $s$ features. 
    \[\hat{\Sigma}^s = \frac{1}{n}\mathbf{Z}^T\mathbf{Z}.\] The eigenvalue matrix is denoted to be $\hat{D}^s = \textnormal{diag}(\hat{\lambda}_1^s,\dots,\hat{\lambda}_N^s)$, where $N = \min\{n,s\}$, and we have \[\lim_{s \rightarrow \infty} \hat{\Sigma}^s = \hat{\Sigma},~~~ \lim_{s \rightarrow \infty} \hat{D}^s = \hat{D}.\]
\end{itemize}

We let $\|A\|$ be the $L_2$ norm if $A$ is a vector and the operator norm if $A$ is an operator. Since $\mathbf{z}_{x}(\mathbf{W})$ is a valid random feature vector approximating the kernel $k$, we use these features to perform standard (linear) ridge regression on these features.

\begin{defn}(Random Feature Ridge Regression) \label{def:random_feature_learning}
Given feature vectors and response variables $\{(\mathbf{z}_{x_i}(\mathbf{W}),y_i)\}_{i=1}^n$, we define:
\begin{itemize}
    \item the random feature regression to be: 
    \begin{IEEEeqnarray}{rCl}
    \beta_{\lambda} := \argmin_{\beta \in \mathbb{R}^{s} }~~ \frac{1}{n}\|Y-\mathbf{Z}\beta\|^2 + \lambda s\|\beta\|^2,\label{rff_opm}
    \end{IEEEeqnarray}
    
    \item the minimum norm least square (MNLS) estimator as:
    \begin{IEEEeqnarray}{rCl}
    \min_{\beta \in \mathbb{R}^s} \|\beta\|^2,~~~ \text{such~that~} \|\mathbf{Z}\beta-Y\|^2 = \min_{\beta_0}\|\mathbf{Z}\beta_0-Y\|^2. \nonumber 
    \end{IEEEeqnarray}   
\end{itemize}
\end{defn}

We denote the RKHS spanned by $\mathbf{Z}$ to be $\tilde{\mathcal{H}}^s$ and we will \textbf{omit} $s$ when the context is clear. By the projection theorem, it is easy to see that the closed form solution of the MNLS estimator is 
\begin{IEEEeqnarray}{rCl}
\tilde{\beta} = \mathbf{Z}(\mathbf{Z}\mathbf{Z}^T)^{\dagger}Y = (\mathbf{Z}^T\mathbf{Z})^{\dagger}\mathbf{Z}^TY, \label{eqn:mnls}
\end{IEEEeqnarray}  
where $A^{\dagger}$ denotes the pseudoinverse for matrix $A$.

\subsection{Bias-Variance Decomposition}
The analysis of the excess learning risk often starts with the bias-variance decomposition. Hence, we present the bias-variance decomposition and introduce some relevant notation here to ease our following discussion. The following lemma gives the bias-variance decomposition and its proof in Appendix \ref{appen:bias-var} is a simple non-linear extension of the one in \cite{hastie2019surprises}. 

% We first recall that we let $f_*(x) = \mathbb{E}(\mathsf{y}| \mathsf{x} = x)$ to be the optimal estimating function. We now define $f_{\tilde{\mathcal{H}}}$ to be the minimal risk estimator for the RKHS $\tilde{\mathcal{H}}$: \[f_{\tilde{\mathcal{H}}} = \inf_{f \in \tilde{\mathcal{H}}} \mathcal{E}(f).\] Assume the unrealizable case ($f_* \notin \tilde{\mathcal{H}}$), for an estimator $\tilde{f}$, we can decompose the excess learning risk as 
% \begin{IEEEeqnarray}{rCl}
% R&(\tilde{f}) &:=  \mathbb{E}_{\pmb{\mathsf{x}}}[(\tilde{f}(\pmb{\mathsf{x}})- f_*(\pmb{\mathsf{x}}))^2|X]   \nonumber \\
%  &\leq& \mathbb{E}_{\pmb{\mathsf{x}}}[(\tilde{f}(\pmb{\mathsf{x}}) - f_{\tilde{\mathcal{H}}}(\pmb{\mathsf{x}}))^2|X] + \mathbb{E}_{\pmb{\mathsf{x}}}[( f_{\tilde{\mathcal{H}}}(\pmb{\mathsf{x}})- f_*(\pmb{\mathsf{x}}))^2|X], \label{eqn:bias_var_unre}
% \end{IEEEeqnarray}
% We refer the first term in Eq.(\ref{eqn:bias_var_unre}) to be the estimation error, denoted as $\mathbf{E}_R$, and the second term as the misspecification error, denoted as $\mathbf{M}_R$. Note that in the realizable case ($f_* \in \tilde{\mathcal{H}}$), $\mathbf{M}_R = 0$ and hence the excess learning risk is purely from the estimation error $\mathbf{E}_R$.

\begin{lma} \label{lma:bia_var}
Let $\tilde{\beta}$ be the MNLS estimator as Eq.(\ref{eqn:mnls}) associated with feature matrix $\mathbf{Z}$. Let $\tilde{f}(x) = \mathbf{z}_x(\mathbf{W})^T\tilde{\beta}$ be the prediction from the MNLS estimator at a test point $x$. Denote $\Pi = (\mathbf{Z}^T\mathbf{Z})^{\dagger}\mathbf{Z}^T\mathbf{Z} - I $ and recall $f_*(x) = \mathbb{E}(\pmb{\mathsf{y}}|\pmb{\mathsf{x}} = x)$, if we assume that $f_* \in \tilde{\mathcal{H}}$ such that $f_*(x) = \mathbf{z}_x(\mathbf{W})^T \beta_*$ for some $\beta_* \in \mathbb{R}^s$, then the following decomposition of the excess risk of $\tilde{\beta}$ holds:
\begin{IEEEeqnarray}{rCl} 
R(\tilde{\beta}) \footnote{Note that we have parametrised the prediction function $\tilde{f}$ by $\tilde{\beta}$, hence we write $R(\tilde{\beta})$ instead of $R(\tilde{f})$.}&:=& R(\tilde{f}) =   \mathbf{B}_{R} + \mathbf{V}_{R}, \nonumber\\
\mathbf{B}_R & = & \mathbb{E}_{\pmb{\mathsf{x}}}\left[\left(\mathbb{E}_{Y|X} [\tilde{f}(\pmb{\mathsf{x}})] -f_*(\pmb{\mathsf{x}})\right)^2\right] = \mathbb{E}_{\pmb{\mathsf{x}}}\|\mathbf{z}_{\pmb{\mathsf{x}}}(\mathbf{W})^T\Pi\beta_*\|^2, \nonumber\\
\mathbf{V}_R &= & \mathbb{E}_{\pmb{\mathsf{x}}} \textnormal{Var}_{Y|X}( \tilde{f}(\pmb{\mathsf{x}}))\nonumber \\
&= & \mathbb{E}_{\pmb{\mathsf{x}}}\left\{\mathbb{E}_{Y|X}\| \mathbf{z}_{\pmb{\mathsf{x}}}(\mathbf{W})^T(\mathbf{Z}^T\mathbf{Z})^{\dagger}\mathbf{Z}^T(Y -f_*(X))\|^2\right\}, \nonumber
\end{IEEEeqnarray}
where $f_*(X) = [f_*(x_1),\cdots, f_*(x_n)]^T$.
\end{lma}

Lemma \ref{lma:bia_var} states that in the realizable case ($f_* \in \tilde{\mathcal{H}}$), the decomposition splits the excess risk into bias ($\mathbf{B}_R$) and variance ($\mathbf{V}_R$). Classical learning theory on bias-variance trade-off \cite[Chapter 2.9]{friedman2001elements} claims that when model is relatively simple, $\mathbf{B}_{R}$ is large but $\mathbf{V}_{R}$ is small. As the model complexity increases, $\mathbf{B}_{R}$ decreases while $\mathbf{V}_{R}$ increases. This forms the familiar U-shape learning curve. This paradigm has been challenged recently \cite{belkin2018understand,belkin2019reconciling,belkin2019two}, and new analysis indicates that the learning curve undergoes the so called double-descent phenomenon. In this contribution, we argue that this phenomenon can be viewed in light of the noisy features, which has been overlooked in the previous literature.\\

It should be noted that the above decomposition applies to the realizable case. One might encounter the unrealizable case where $f_* \notin \tilde{\mathcal{H}}$. In this case, if we denote the best predictor from $\tilde{\mathcal{H}}$ to be $f_{\tilde{\mathcal{H}}}$, then the difference between $f_{\tilde{\mathcal{H}}}$ and $f_*$ also contributes to the excess learning risk. We denote this additional term in excess risk by $\mathbf{M}_R$ (which can be thought of as the approximation error). The excess learning risk in the unrealizable case is hence comprised of $\mathbf{M}_R$, $\mathbf{B}_R$ and $\mathbf{V}_R$. While the $\mathbf{M}_R$ is new, the analysis of $\mathbf{B}_R$ and $\mathbf{V}_R$ is the same as in the realizable case. Our next lemma illustrates this (Proof in Appendix \ref{appen:bias-var-unrea}):
\begin{lma} \label{lma:bia_var_unrea}
Let $\tilde{\beta}$ be the MNLS estimator as Eq.(\ref{eqn:mnls}) associated with feature matrix $\mathbf{Z}$. Let $\tilde{f}(x) = \mathbf{z}_x(\mathbf{W})^T\tilde{\beta}$ be the prediction from the MNLS estimator at a test point $x$. Define $f_{\tilde{\mathcal{H}}}$ to be the best estimator such that $f_{\tilde{\mathcal{H}}} := \argmin_{f \in \tilde{\mathcal{H}}} \mathbb{E}_{\rho}(f(x)-y)^2$. Since $f_{\tilde{\mathcal{H}}} \in \tilde{\mathcal{H}}$, we let $f_{\tilde{\mathcal{H}}} (x) = \mathbf{z}_x(\mathbf{W})^T \beta_{\tilde{\mathcal{H}}}$ for some $\beta_{\tilde{\mathcal{H}}} \in \mathbb{R}^s$, Define the bias and the variance as: 
\begin{IEEEeqnarray}{rCl} 
\mathbf{B}_R & = & \mathbb{E}_{\pmb{\mathsf{x}}}\left[\left(\mathbb{E}_{Y|X} [\tilde{f}(\pmb{\mathsf{x}})] -f_{\tilde{\mathcal{H}}}(\pmb{\mathsf{x}})\right)^2\right] = \mathbb{E}_{\pmb{\mathsf{x}}}\|\mathbf{z}_{\pmb{\mathsf{x}}}(\mathbf{W})^T\Pi\beta_{\tilde{\mathcal{H}}}\|^2, \nonumber\\
\mathbf{V}_R &= & \mathbb{E}_{\pmb{\mathsf{x}}} \textnormal{Var}_{Y|X}( \tilde{f}(\pmb{\mathsf{x}}))\nonumber \\
&= & \mathbb{E}_{\pmb{\mathsf{x}}}\left\{\mathbb{E}_{Y|X}\| \mathbf{z}_{\pmb{\mathsf{x}}}(\mathbf{W})^T(\mathbf{Z}^T\mathbf{Z})^{\dagger}\mathbf{Z}^T(Y -f_{\tilde{\mathcal{H}}}(X))\|^2\right\}, \nonumber
\end{IEEEeqnarray}
where $f_{\tilde{\mathcal{H}}}(X) = [f_{\tilde{\mathcal{H}}}(x_1),\cdots, f_{\tilde{\mathcal{H}}}(x_n)]^T$. In addition, we define the misspecification as:
\[\mathbf{M}_R := \mathbb{E}_{x} \left\{\mathbf{z}_{x}(\mathbf{w})^T(\mathbf{Z}^T\mathbf{Z})^{\dagger}\mathbf{Z}^T \left(f_*(X) - f_{\tilde{\mathcal{H}}}(X) \right) \right\}^2 + \mathbb{E}_{x} (f_*(x) - f_{\tilde{\mathcal{H}}}(x))^2. \]
Then the following decomposition of the excess risk of $\tilde{\beta}$ holds:
\begin{IEEEeqnarray}{rCl} 
R(\tilde{\beta})  \leq   3 (\mathbf{M}_R + \mathbf{B}_{R} + \mathbf{V}_{R}). \nonumber
\end{IEEEeqnarray}
\end{lma}
We can see that the bias and variance comprised in the unrealizable case is similar to the realizable case. However, the difference now is that we have $\mathbf{M}_R$, the misspecification error. While we do not have a thorough understanding as how $\mathbf{M}_R$ evolves with the number of parameters $s$, we can reasonably assume that the $\mathbf{M}_R$ decreases as we increase $s$. In addition, the $\mathbf{M}_R$ becomes zero once $\tilde{\mathcal{H}}$ is large enough to contain $f_*$. As a result, our following analysis will mainly focus on the $\mathbf{B}_R$ and $\mathbf{V}_R$ term as these two will dominate in the large parameter setting. For $\mathbf{M}_R$, we will simply assume that it decreases with the number of parameters deployed in the model.

\section{Benign Overfitting with Noisy Random Features}
In this section, we discuss how the behaviour of the excess learning risk of the MNLS estimator is affected by the noise in the features and how the new evolution of the excess learning risk leads to benign overfitting and, in particular, to the double descent phenomenon. In the following discussion, we let $P > p > n,s$ without loss of generality. In addition, since we are mainly interested in the overparametrized regime, we let $s \geq n$.

As discussed, we consider the noisy feature setting: $z_{\xi}(x,w) = z(x,w) + \xi$. We denote the $s$-dimensional noisy feature as $\mathbf{z}^{\xi}_x(\mathbf{W}) = \mathbf{z}_x(\mathbf{W}) + \pmb{\xi}$, where $\pmb{\xi} = [\xi_1,\dots,\xi_s]^T$ with $\xi_1,\dots,\xi_s$ i.i.d $\sim$ $\xi$. In addition, recall we define the feature matrix as $\mathbf{Z} = [\mathbf{z}_{x_1}(\mathbf{W}),\dots,\mathbf{z}_{x_n}(\mathbf{W})]^T$. In the noisy setting, we write the noisy feature matrix as $\mathbf{Z}_{\xi} = \mathbf{Z} + \Xi$, where $\Xi = [\xi_{ij}] \in \mathbb{R}^{n\times s}$ with each $\xi_{ij}$ i.i.d $\sim$ $\xi$. We let $\tilde{\mathcal{H}}_{\xi}$ to be the RKHS spanned by the noisy feature matrix $\mathbf{Z}_{\xi}$. Finally, similar to Eq.(\ref{eqn:mnls}), we write the MNLS estimator for $\mathbf{Z}_{\xi}$ as $\tilde{\beta}_{\xi}$.

We first list our assumptions (which we will use throughout the paper):
\begin{enumerate}
    \item[A.$1$] The RKHS Condition: Assume $P = \infty$ and $\int_{\mathcal{X}}K(x,x)~ d\rho_{\mathcal X}(x) = C_0$ ($0 < C_0 < \infty$);
   
    \item[A.$2$] Label Noise Condition: $\mathcal{X} \subset \mathbb{R}^d$, and $y = f_*(x) + \epsilon$ with $\mathbb{E}(\epsilon) = 0$ and $\text{Var}(\epsilon) = \sigma^2$;
    
    \item[A.$3$] Best Predictor Condition: $f_* \in \tilde{\mathcal{H}}_{\xi}$, the best predictor is contained in the hypothesis space and has the form of $f_*(x) =\mathbf{z}^{\xi}_x(\mathbf{W})^T\beta_*^{\xi}$;
    
    \item[A.$4$] Feature Noise Condition: $\xi \sim \mathcal{N}(0,\frac{1}{s}\sigma_0^2)$ and $\sigma_0^2 = s^{-\alpha}$, with $\alpha \geq 0$.
\end{enumerate}

Assumption A.1 is a weak condition to ensure $K$ is trace class, i.e., $\textnormal{Tr}(\Sigma_K) < \infty$, and admits Mercer's decomposition. This further implies the Mercer's decomposition for the low rank kernel $k$. A.2 is a standard regression assumption. A.3 assumes we are in the realizable case which is a relatively strong assumption when number of features $s$ is small. We will discuss the unrealizable case in Section \ref{sec:double_descent}. A.4 describes the shape and size of the feature noise $\xi$, note that we need the variance to be $\frac{1}{s}\sigma_0^2$ to ensure that the variance of the feature vector does not explode as $\text{Var}(\mathbf{z}^{\xi}_x(\mathbf{W})) \geq \text{Var}(\pmb{\xi}) = s\text{Var}(\xi) = \sigma_0^2 $.

In the noisy setting, the excess learning risk $R(\tilde{\beta}_{\xi})$ admits the bias-variance decomposition similar to Lemma \ref{lma:bia_var}. We will denote the bias and variance term as $\mathbf{B}_{\xi}$ and $\mathbf{V}_{\xi}$ respectively. We are now ready to present our analysis of the excess learning risk in the noisy feature regime.
\begin{theorem}\label{theo: over_rff_noisy}
Under Assumptions A1-A4 and suppose we are in the overparametrized regime where $s \geq n$, denote $\Pi_{\xi} = (\mathbf{Z}_{\xi}^T\mathbf{Z}_{\xi})^{\dagger}\mathbf{Z}_{\xi}^T\mathbf{Z}_{\xi} - I$. Recall $\mathbf{W}$ is a $p \times s$ matrix with each $\mathbf{W}_{ij}$ i.i.d $\sim \mathcal{N}(0,1)$, let $\lambda_{W} = \left\|\mathbf{W}^T\mathbf{W}\right\|$ and $a, b,c > 1$ be some universal constants. Denote $\hat{\lambda}_{i}^{\xi} = \hat{\lambda}_i + \sigma_0^2/n$, then if we assume that there exists $k^*$ defined as:
\begin{eqnarray}
k^* = \min\left\{0\leq k \leq n,  \sum_{i>k}^n \frac{\hat{\lambda}_{i}^{\xi}}{\hat{\lambda}_{k+1}^{\xi}} \geq \frac{1}{a}n\right\}, \label{assum_noisy}
\end{eqnarray}
For any $\delta \in (0,1)$, with probability at least $1-\delta-6e^{-n/b}-5e^{-n/c}$, we have 
\begin{eqnarray}
\mathbf{B}_{\xi} &\leq &  b \left(\frac{\lambda_W}{s} \|\Sigma\|\sqrt{\log(\frac{14r(\Sigma)}{\delta})/n} + \sigma_0 + \sigma_0^2 \right) \|\Pi_{\xi}\|^2\|\beta_{*}^{\xi}\|^2, \label{noi_bia_up}\\
\mathbf{V}_{\xi} &\leq& c\sigma^2\textnormal{Tr}(\Sigma) \frac{s}{n^2}, \label{noi_var_up}
\end{eqnarray}
where $r(\Sigma) = \textnormal{Tr}(\Sigma)/\|\Sigma\|$ .
\end{theorem}

Theorem \ref{theo: over_rff_noisy} explains precisely how $\xi$ affects the excess learning risk. In fact, the upper bound on the bias and variance term serve as the certificate that asymptotically the MNLS estimator $\tilde{\beta}_{\xi}$ will obtain optimal prediction accuracy, i.e., the excess risk $\mathbf{B}_{\xi} + \mathbf{V}_{\xi}$ can decay to $0$.

Before we analyze the asymptotic behaviour of the excess risk, we would like to first discuss our key assumption: Eq.(\ref{assum_noisy}). There are two scenarios here: $\alpha = 0$ and $\alpha >0$. We start with the first one. In this case, $\sigma_0^2 = 1$ is a constant. Eq.(\ref{assum_noisy}) states that there is a $k^*$ such that we have $\sum_{i> k^*}^n(\hat{\lambda}_i^{\xi}/\hat{\lambda}_{k^*+1}^{\xi}) \geq \frac{1}{a}n$. This is equivalent to \[\sum_{i> k^*}^n \frac{n\hat{\lambda}_i + \sigma_0^2}{n\hat{\lambda}_{k^*+1}+\sigma_0^2} =\sum_{i> k^*}^n \frac{n\hat{\lambda}_i + 1}{n\hat{\lambda}_{k^*+1}+1} \geq \frac{1}{a}n.\]
Now, if $n\hat{\lambda}_{k^*+1} \leq 1$, then we have for each $i> k^*$, $(n\hat{\lambda}_i + 1)/(n\hat{\lambda}_{k^*+1}+1) \geq \frac{1}{2}$. As a result, $\sum_{i> k^*}^n(\hat{\lambda}_i^{\xi}/\hat{\lambda}_{k^*+1}^{\xi}) \geq \frac{1}{2}(n-k^*)$. If $k^* \ll n$, then there is some universal constant $a$, such that $\frac{1}{2}(n-k^*) \geq \frac{1}{a}n$. To conclude, if there exists $k^* \ll n$ and $n\hat{\lambda}_{k^*+1} \leq 1$, then $\sum_{i> k^*}^n(\hat{\lambda}_i^{\xi}/\hat{\lambda}_{k^*+1}^{\xi}) \geq \frac{1}{a}n$. On the other hand, both $\text{Tr}(\Sigma) < \infty$ and $\text{Tr}(\hat{\Sigma}) < \infty$, implying $\hat{\lambda}_k = \omega_1 k^{-\gamma}$ for some constant $\omega_1$ and $1 < \gamma \leq \infty$. Based on different values of $\gamma$, there are three different cases here.
\begin{itemize}
    \item[Case $1$] $\gamma = \infty$: $\hat{\Sigma}$ has finite rank, so there is some $d$ such that $\hat{\lambda}_i = 0$ for $i > d$. As such, if we let $k^* = d$, we can easily see that $\sum_{i> k^*}^n(\hat{\lambda}_i^{\xi}/\hat{\lambda}_{k^*+1}^{\xi}) = (n-d) \geq \frac{1}{a}n$.
    
    \item[Case $2$] $\gamma \propto k$: $\hat{\Sigma}$ has exponential spectrum decay, i.e., $\hat{\lambda}_k = \omega_1 e^{-k}$. Without loss of generality we assume $\omega_1 =1$, then if we let $k^* = \log n$, it is easy to see that $n\hat{\lambda}_{k^*+1} = \frac{n}{n+1}\leq \ 1$. We hence have $\sum_{i> k^*}^n(\hat{\lambda}_i^{\xi}/\hat{\lambda}_{k^*+1}^{\xi}) = \frac{1}{2}(n-\log n) \geq \frac{1}{a}n$.
    
    \item[Case $3$] $\gamma$ is a constant: $\hat{\Sigma}$ has polynomial decay, i.e., $\hat{\lambda}_k = \omega_1 k^{-\gamma}$. Again we assume $\omega_1 = 1$ and if we let $k^* = n^{1/\gamma}$, we have $\hat{\lambda}_{k^*+1} = (n^{1/\gamma}+1 )^{-\gamma}\leq  (n^{1/\gamma} )^{-\gamma}\leq 1$. Therefore, we have $\sum_{i> k^*}^n(\hat{\lambda}_i^{\xi}/\hat{\lambda}_{k^*+1}^{\xi}) \leq \frac{1}{2}(n- n^{1/\gamma}) \geq \frac{1}{a}n$.
\end{itemize}

The analysis in the second scenario where $\alpha > 0$ is similar to the first one. The key is that there exists $k^* \ll n$ such that $n\hat{\lambda}_{k^*+1} \leq \sigma_0^2$. The difference here is that $\sigma_0^2$ decays with $s$ at rate $\alpha > 0$. However, if we control the decay rate $\alpha$ such that $\alpha \ll \gamma$, then we can guarantee the existence of $k^*$. In summary, as long as $\alpha \ll \gamma$, then we can find a $k^*$ such that Eq.(\ref{assum_noisy}) holds.

We are now ready to analyze the asymptotic behaviour of the excess risk. We start with $\mathbf{V}_{\xi}$ where we can see that the variance $\mathbf{V}_{\xi}$ is governed by $\frac{s}{n^2}$. Hence, if we let $s = o(n^{2})$, then we have $\lim_{n\rightarrow \infty}\mathbf{V}_{\xi} = 0$. For the bias $\mathbf{B}_{\xi}$, $\sigma_0$ and $\sigma_0^2$ decays to $0$ as long as $\alpha > 0$. In addition, it is easy to see that $\lambda_W = O(p)$, hence if we have $\lim_{n \rightarrow} \frac{p}{s}\sqrt{\frac{1}{n}} = 0$, then $\lim_{n\rightarrow \infty} \mathbf{B}_{\xi}= 0$.

Therefore, the following two conditions are required to ensure that our predictor $\tilde{\beta_{\xi}}$ is optimal:
\begin{itemize}
    \item $\lim_{n\rightarrow \infty} \frac{p}{s}\sqrt{\frac{1}{n}} = 0$;
    \item $s = o(n^2)$.
\end{itemize}
We can see that if we let $s= n^{\gamma_0}$ for some $\gamma_0 \in (1,2)$, then $s \gg n$ since $\lim_{n\rightarrow \infty} s/n = \infty$. In other words, even in the heavily overparametrized model, our estimator $\tilde{\beta}_{\xi}$ is optimal since both $\mathbf{B}_{\xi}$ and $\mathbf{V}_{\xi}$ converge to $0$. However, our theorem indicates that once $s$ is beyond the order of $n^2$, the variance start to increase again. This result is aligned with the recent discovery by \cite[Figure 1 and 4]{adlam2020neural}, where the excess risk is found to increase if $s$ is close to the order of $n^2$.

\subsubsection*{Motivation} Before providing a sketch proof of Theorem \ref{theo: over_rff_noisy}, we first state our motivation for considering noisy feature $z_{\xi}(x,w)$. The motivation arises from analyzing the excess risk of the MNLS estimator in the noiseless version, i.e., $\tilde{\beta}$ in Eq.(\ref{eqn:mnls}). Proposition \ref{theo:over_rff} establishes nearly matching upper and lower bounds for the excess risk of $\tilde{\beta}$. The proof is in Section \ref{sec:pf_over_rff}. 

\begin{prop}\label{theo:over_rff}
Consider the regression problem Eq.(\ref{rff_opm}) with feature matrix $\mathbf{Z}$ and suppose we are in the overparametrized regime where $s \geq n$. Let $a,b,c,c' > 1$ be some universal constants and assume $s \geq n$. Let $\delta \in (0,1)$ and denote \[k^* = \min \left\{0 \leq k \leq n, \frac{\sum_{i > k}^n \hat{\lambda}_i }{\hat{\lambda}_{k+1}} \geq \frac{1}{a} n \right\}.\] Then with probability greater than $1-\delta-2e^{-n/c}$, we have
\begin{eqnarray}
R(\tilde{\beta}) &=& \mathbf{B}_R + \mathbf{V}_R, \nonumber\\
&\leq& b\frac{\lambda_{W}}{s} \|\Pi\|^2\|\beta_*\|^2\|\Sigma\|\sqrt{\log\left(\frac{14r(\Sigma)}{\delta}\right)/n} + c\sigma^2\frac{s}{n} \frac{\textnormal{Tr}(\Sigma)}{\sum_{i>k^*}^n\hat{\lambda}_i},\label{risk_ns}
\end{eqnarray}
 We can also lower bound the risk with probability greater than $1-2e^{-n/c'}$:
\begin{eqnarray}
R(\tilde{\beta})  \geq c'\sigma^2\frac{s}{n}  \frac{\textnormal{Tr}(\Sigma)}{\sum_{i>k^*}^n\hat{\lambda}_i}. \label{risk_low}
\end{eqnarray}
\end{prop}

Inspecting Eq.(\ref{risk_ns}), we can see that as long as $\lim_{p,s \rightarrow \infty} \frac{p}{s}\sqrt{\frac{1}{n}} = 0$, the bias term $\mathbf{B}_{R}$ decays (since $\lambda_{W} = O(p)$). Therefore, if the variance term decays, then $R(\tilde{\beta})$ decreases to $0$. As a result, Proposition \ref{theo:over_rff} states that we need the following conditions for $\tilde{\beta}$ to have optimal prediction accuracy:
\begin{enumerate}
    \item The covariance operator is of trace-class;
    \item The sum of the tail eigenvalues of $\hat{\Sigma}$ is on the order of $N$, i.e., there exists a $k^*$ such that $\sum_{i>k^*}^n\hat{\lambda}_i = \Theta(n)$ and $\lim_{s,n \rightarrow \infty} \frac{s}{n}\frac{1}{\sum_{i>k^*}^n\hat{\lambda}_i} = 0$.
\end{enumerate}
The first condition is a standard requirement for a typical learning problem, where the last condition states that we need $s$ to be of order $o(n^2)$. While these two conditions seems reasonable, the second condition seems at odds with the first condition. Namely, according to the classical concentration inequality, $\|\Sigma - \hat{\Sigma}\| \rightarrow 0$ as $n \rightarrow \infty$, we hence have $\sum_{i>k}^n \hat{\lambda}_i \leq \text{Tr}(\hat{\Sigma})$ which is finite and does not grow with $n$. However, traditional concentration theory requires that the observed samples $\{x_i\}_{i=1}^n$ are i.i.d from the marginal distribution $\rho_{\pmb{\mathsf x}}(x)$. Hence, $\hat{\Sigma} = \frac{1}{n}\sum_{i=1}^n D^{1/2}V_{x_i}V_{x_i}^TD^{1/2}$ is simply an empirical estimate of $\Sigma = \int D^{1/2}V_xV_x^TD^{1/2} d\rho(x)$. However, if the feature $z(x,w)$ is corroded with noise $\xi$, then this will distort the behaviour of $\hat{\Sigma}$. We qualitatively discuss the effect of $\xi$ below, in order to provide an intuition on how benign overfitting arises.\\

Recall the definitions of $\Sigma$, $\hat{\Sigma}$ and $\hat{\Sigma}^s$ in Section \ref{sec:new_look}, in the noiseless setting, $\hat{\Sigma}^s = \frac{1}{n} \mathbf{Z}^T\mathbf{Z} \in \mathbb{R}^{s\times s}$. As $s \rightarrow \infty$, $\hat{\Sigma}_s \rightarrow \hat{\Sigma}$, implying $\hat{D}^s \rightarrow \hat{D}$. In the noisy setting, suppose the feature matrix is corroded with some i.i.d noise: $\mathbf{Z}_{\xi} = \mathbf{Z} + \Xi$, the covariance matrix now is \[\hat{\Sigma}_{\xi}^s = \frac{1}{n}\mathbf{Z}_{\xi}^T\mathbf{Z}_{\xi} = \frac{1}{n}(\mathbf{Z}^T\mathbf{Z}+ \Xi^T\mathbf{Z} +\mathbf{Z}^T\Xi  + \Xi^T\Xi).\] Denote $\hat{D}_{\xi}^s$ to be the eigenvalues of $\hat{\Sigma}_{\xi}^s$. As $s \rightarrow \infty$, we approximately have $\hat{D}_{\xi}^s  \approx \textnormal{diag}(\hat{\lambda}^s_1 + \sigma_0^2, \dots, \hat{\lambda}^s_n + \sigma_0^2)$. Since $\hat{\lambda}_i^s$ decays, there will be a $k^* < n,s$ such that $\hat{\lambda}_i^s > \sigma_0^2, \forall i < k^*$. However, $\hat{\lambda}^s_i$ is on the same scale of $\sigma_0^2$ for all $i > k^*$, in the sense that \[\frac{\hat{\lambda}_i^s + \sigma_0^2}{\hat{\lambda}_j^s + \sigma_0^2} \approx \Theta(1), \text{for~all~} i, j > k^*.\] Since $\hat{\Sigma}_{\xi}^s$ has at most $n$ eigenvalues, summing up the tails gives \[\frac{\sum_{i>k^*}^n (\hat{\lambda}_i^s + \sigma_0^2)}{\hat{\lambda}_{k^*+1}^s + \sigma_0^2} \approx \Theta(n).\] This indicates that the sum of the tail eigenvalues of the covariance matrix $\hat{\Sigma}_{\xi}^s$ is of the order of $n$, leading to the decay of $R(\tilde{\beta})$. This further motivates us to quantify how exactly the noise $\xi$ affect the behaviour of the excess learning risk in Theorem \ref{theo: over_rff_noisy}. Below we provide a sketch of the proof for Theorem \ref{theo: over_rff_noisy}.

\begin{proof}(Sketch of Proof for Theorem \ref{theo: over_rff_noisy})
The proof starts with the Bias-Variance decomposition. We employ the noisy feature version of Lemma \ref{lma:bia_var}, where the excess learning risk is decomposed to the bias term $\mathbf{B}_{\xi}$ and the variance term $\mathbf{V}_{\xi}$. While the treatment to $\mathbf{B}_{\xi}$ is relatively standard, the heavy part is on how to analyze $\mathbf{V}_{\xi}$. The key is to express the $\mathbf{V}_{\xi}$ to be a sum of the outer product of random vectors with each entry being i.i.d standard Gaussian random variables. After that, we apply concentration inequalities to the outer products which gives us the desired results. For detailed derivation, please refer to Section \ref{sec:gau_proof}. 
\end{proof}

\subsection{Benign Overfitting with Subgaussian Noisy Features}
Theorem \ref{theo: over_rff_noisy} demonstrates that if $\xi$ is Gaussian with decaying variance, then benign overfitting can be observed. However, Gaussian noise is sometimes a strong assumption. A close investigation of Theorem \ref{theo: over_rff_noisy} indicates that the key driving force of benign overfitting is that $\sigma_0^2$ decays with $s$, not the shape of $\xi$. Hence, we conjecture that benign overfitting will occur even if we have non-Gaussian distributions. It turns out that a simple extension of Theorem \ref{theo: over_rff_noisy} would allow us to generalize our results to subgaussian noise. Hence, we modify Assumption A.$4$ to  

\begin{itemize}
    \item[A.$4'$] Feature Noise Condition: $\xi$ is a subgaussian in the sense that $\xi = \frac{\sigma_0^2}{s} u$, where $u$ is mean $0$, variance $1$ and $\sigma_u^2$-subgaussian, i.e., \[\mathbb{E}(\exp(tu)) \leq \exp\left( \frac{\sigma_u^2}{2}t^2 \right).\]
\end{itemize}

Our results below confirm that benign overfitting can indeed also be observed in the subgaussian noise setting.
\begin{cor}\label{theo: over_rff_noisy_sug}
Under A.$1$-$4'$ and suppose $s \geq n$, let $a, b,c > 1$ be some universal constants, recall $\hat{\lambda}_{i}^{\xi} = \hat{\lambda}_i + \sigma_0^2/n$, if we assume that there exists $k^*$ defined as: 
\begin{eqnarray}
k^* = \min\left\{0\leq k \leq n,  \sum_{i>k}^n \frac{\hat{\lambda}_{i}^{\xi}}{\hat{\lambda}_{k+1}^{\xi}} \geq \frac{1}{a}n\right\}, \label{assum_sub}
\end{eqnarray}
For any $\delta \in (0,1)$, with probability at least $1-\delta-6e^{-n/b}-5e^{-n/c}$, we have 
\begin{eqnarray}
\mathbf{B}_{\xi} &\leq &  b \left(\frac{\lambda_W}{s} \|\Sigma\|\sqrt{\log(\frac{14r(\Sigma)}{\delta})/n} + \sigma_0 + \sigma_0^2 \right) \|\Pi_{\xi}\|^2\|\beta_{*}^{\xi}\|^2, \label{noi_bia_up_sub}\\
\mathbf{V}_{\xi} &\leq& c\sigma^2\textnormal{Tr}(\Sigma) \frac{s}{n^2}.  \label{noi_var_up_sub}
\end{eqnarray}
\end{cor}

The behaviour of the excess learning risk in the subgaussian case is almost identical to the Gaussian case up to some constant. As discussed in the Gaussian case, this leads to the decaying of the learning risk asymptotically. The results verify our conjecture that as long as the noise $\xi$ decays with $s$, we will observe benign overfitting. Corollary \ref{theo: over_rff_noisy_sug} further confirms that the noise $\xi$ in the covariate or feature vector can serve as an implicit regularizer to prevent overfitting.

\subsection{The Double Descent Phenomenon} \label{sec:double_descent}
The classical U-shape learning curve \cite[Figure 2.11]{friedman2001elements} has been largely challenged recently \cite{belkin2018understand,belkin2019reconciling,belkin2019two}, as empirically it is often observed that the relationship between the prediction accuracy and the complexity of the learning machine exhibits the so called the double descent phenomenon. Looking Figure \ref{fig:double_descent} \cite{belkin2019reconciling}, the double descent curve states that when we first increase the capacity of the hypothesis space $\mathcal{H}$, the excess learning risk decreases but starts to increase as we keep increasing the model complexity. The excess learning risk increases to the maximum (or potentially diverges to infinity) at some interpolation threshold. After that, as we keep increasing the complexity of $\mathcal{H}$, the excess learning risk decreases either to a global minimum or vanishes to zero. Overall, it forms a double descent curve. 

\begin{figure}[ht]
    \centering
    \includegraphics[width = 13cm, height = 6.0cm]{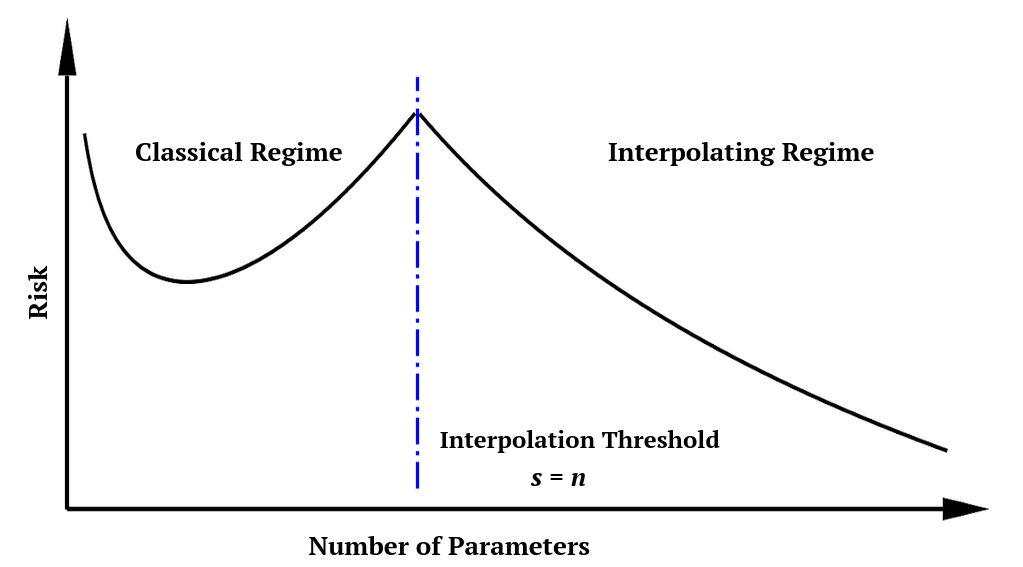}
    \caption{The \textit{double descent curve}.}
    \label{fig:double_descent}
\end{figure}
The double descent phenomenon has attracted much research interest recently while a concluding answer has not been discovered in general. In this section, we try to provide an answer to how double descent occurs as a result of noisy features through analyzing the behaviour of the excess learning risk in Corollary \ref{theo: over_rff_noisy_sug}.\\  

Before delivering our results, we first state our notations and assumptions to ease discussion. Recall in Lemma \ref{lma:bia_var}, we have decomposed the excess learning risk into the misspecification error $\mathbf{M}_R$ (or $\mathbf{M}_{\xi}$ in the noisy feature setting), the bias $\mathbf{B}_R$ (or $\mathbf{B}_{\xi}$) and the variance $\mathbf{V}_R$ (or $\mathbf{V}_{\xi}$). Connection between Corollary \ref{theo: over_rff_noisy_sug} and the double descent is through analyzing these errors in the noisy feature setting: $\mathbf{M}_{\xi}$, $\mathbf{B}_{\xi}$ and  $\mathbf{V}_{\xi}$. We now demonstrate the finite sample behaviour of the excess learning risk from Corollary \ref{theo: over_rff_noisy_sug}.

\begin{cor}\label{theo:double_descent}
Under Assumptions A.$1$-$4'$, the behaviour of the excess learning risk $R(\tilde{\beta}_{\xi})$ can be described as the following:
\begin{itemize}
        \item[a.] $s = n$ 
        \begin{IEEEeqnarray}{rCl} 
        \mathbf{B}_{\xi} &=& O(\frac{p}{n^{3/2}})+ O(n^{-\alpha}),\nonumber \\
        \mathbf{V}_{\xi} &=& O(n^{-1});\nonumber
        \end{IEEEeqnarray}
        \item[b.] $s = o(n^{\gamma_1})$, $\gamma_1 \in (1,2)$
        \begin{IEEEeqnarray}{rCl} 
        \mathbf{B}_{\xi} &=& O(\frac{p}{s}n^{-\frac{1}{2}})+O(n^{-\alpha}),\nonumber \\
        \mathbf{V}_{\xi} &=& O(n^{(\gamma_1-2)});\nonumber
        \end{IEEEeqnarray}
        \item[c.] $s = \Theta( n^{\gamma_2})$, $\gamma_2 > 2$ \[\mathbf{V}_{\xi} = \Theta(n^{(\gamma_2 -2)}).\]
\end{itemize}
\end{cor}

Corollary \ref{theo:double_descent} describes the precise behaviour of the excess learning risk in the overparametrized regime with $s \geq n$. Before we give a detailed discussion on that, we first qualitatively discuss the behaviour of the excess learning risk in the underparametried regime i.e., the first U-shape in Figure \ref{fig:double_descent}. When $s < n$, Corollary \ref{theo: over_rff_noisy_sug} indicates that $\mathbf{B}_{\xi} = 0$ by Lemma \ref{lma:proj}. However, Corollary \ref{theo: over_rff_noisy_sug} assume we are in the realizable case ($f_*\in \tilde{\mathcal{H}}$). When $s$ is small, this is unlikely to happen. As a result, we have the misspecification error $\mathbf{M}_{\xi}$. In another words, in the finite sample case where $s < n$, the excess learning risk is governed by the misspecification error and the variance. Intuitively we can see that $\mathbf{M}_{\xi}$ decreases as we increase $s$, and typically, when $s$ is small, $\mathbf{M}_{\xi}$ dominates the excess learning risk. As we increase $s$, $\mathbf{M}_{\xi}$ decreases and $\mathbf{V}_{\xi}$ increases up to some point, where $\mathbf{V}_{\xi}$ start to dominate the excess learning risk. Therefore, we will observe that the excess learning risk decreases with $s$ initially and after some point, it starts to increase with $s$, which forms the classical U-shape curve.

When we approach the interpolation threshold where $s =n$, Corollary \ref{theo:double_descent} shows that the bias term $\mathbf{B}_{\xi}$ starts to kick in and dominates the excess learning risk by the term $O\left(\frac{p}{n^{3/2}}\right)$, because $n,s \ll p$ at this point. In particular, if we use kernel $K$ where $P = \infty$, then the bias $\mathbf{B}_R$ diverges to infinity.

Furthermore, if we keep increasing $s$ so that it passes the interpolation threshold, the excess learning risk is now dominated by $\mathbf{B}_{\xi}$ and $\mathbf{V}_{\xi}$, since the function space is large enough such that the misspecification error is negligible. As discussed earlier, if $\frac{p}{s\sqrt{n}}$ converges to $0$, and $s = o(n^2)$, then both terms vanish to zero asymptotically, driving the learning risk to its global minimum. In particular, if $s = n^{\gamma}$ with $\gamma \in (1,2)$, then the overfitted model with $s \gg n$ can still have excess learning risk to converge. In addition, if we keep increase $s$ such that it is beyond the order of $n^2$, we can see that the variance start to increase again. Overall, Corollary \ref{theo:double_descent} gives us the precise description of the double descent curve up to the point where $s$ is within the $n^2$ order. This is aligned with the recent findings from \cite[Figure 1 and 4]{adlam2020neural}. 

In Figure \ref{fig:bound_double_descent}, we give an sample path of how our upper bound evolves with the number of features $s$. We can see that it closely resembles the double descent curve in Figure \ref{fig:double_descent} from \cite{belkin2019reconciling}. %Note that while plotting Figure \ref{fig:bound_double_descent}, we have assumed that $\mathbf{M}_{\xi}$ decays with $s$ at order $O(s^{-1})$ for simplicity. If we vary the decay rate, the shape of the curve is similar.}

\begin{figure}[ht]
    \centering
    \includegraphics[width = 13cm, height = 6cm]{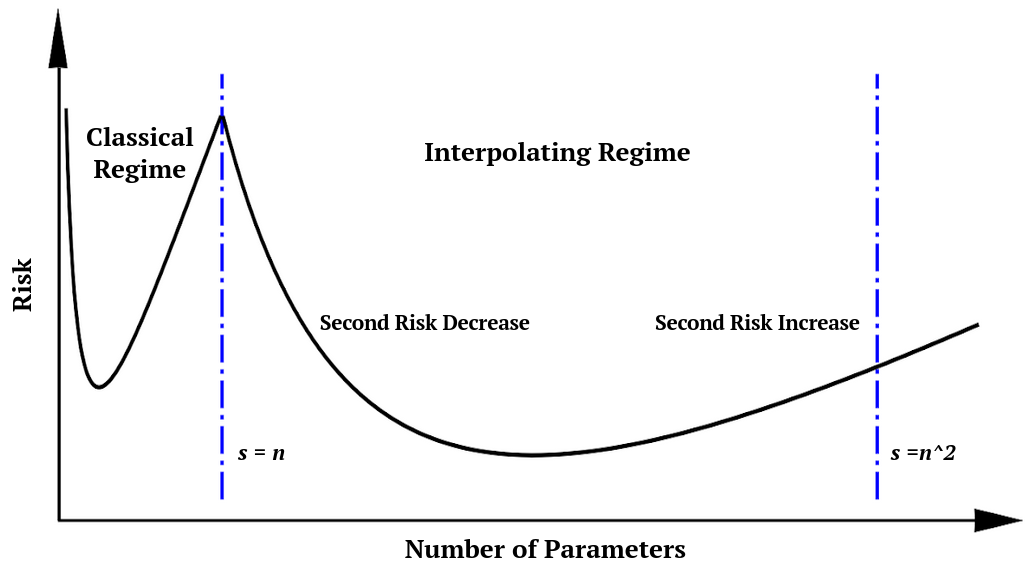}
    \caption{The evolution of the excess learning risk from Corollary \ref{theo:double_descent}. Note that our chracterization has the property where the excess learning risk starts to grow beyond the $n^2$ regime. This is aligned with the recent empirical observation of the triple-double descent learning curve from \cite[Figure 1 and 4]{adlam2020neural}. }
    \label{fig:bound_double_descent}
\end{figure}

\section{Proofs}
In the proof, we use $b_1,b_2,\dots \pmb{> 1}$ to denote universal constants in the theorem statement, and we use $c_1,c_2,\dots \pmb{> 1}$ to denote universal constants in the proof of the theorems.
\subsection{Proof of Proposition \ref{theo:over_rff}} \label{sec:pf_over_rff}
The proof of Proposition \ref{theo:over_rff} starts with the analysis of each term in the bias-variance decomposition. We first deal with the bias term $\mathbf{B}_{R}$.
\subsubsection*{Upper Bound of Bias:} The following lemma gives the upper bound on the bias term $\mathbf{B}_R$.
\begin{lma}\label{lma:bia_up}
The bias can be upper bounded as:
\begin{IEEEeqnarray}{rCl}
\int_{\mathcal{X}} \left\|\mathbf{z}_x(\mathbf{W})^T\Pi\beta_*\right\|^2 d\rho(x) \leq \frac{\lambda_{W}}{s}\|\Pi\| \|\Sigma-\hat{\Sigma}\| \|\beta_*\|^2.\nonumber
\end{IEEEeqnarray}
\end{lma}
\begin{proof}
Using the property of pseudoinverse yields \[\mathbf{Z}^T\mathbf{Z} \Pi =\mathbf{Z}^T\mathbf{Z}(I -(\mathbf{Z}^T\mathbf{Z})^{\dagger}\mathbf{Z}^T\mathbf{Z}) = 0.\]
Hence, we have 
\begin{IEEEeqnarray}{rCl}
\mathbf{B}_{R} &=&  \mathbb{E}_x\left\{ \beta_*\Pi\mathbf{z}_x(\mathbf{W})\mathbf{z}_x(\mathbf{W})^T\Pi\beta_*\right\}= \beta_*\Pi\left(\frac{1}{s}\mathbf{W}^T\Sigma\mathbf{W}\right)\Pi\beta_*,\nonumber\\
&=& \beta_*\Pi\left(\frac{1}{s}\mathbf{W}^T\Sigma\mathbf{W} - \frac{1}{n}\mathbf{Z}^T\mathbf{Z}\right)\Pi\beta_*= \beta_*\Pi\left(\frac{1}{s}\mathbf{W}^T\Sigma\mathbf{W} - \frac{1}{s}\mathbf{W}^T\hat{\Sigma}\mathbf{W}\right)\Pi\beta_*,\nonumber\\
& =& \frac{1}{s}\beta_*\Pi\left(\mathbf{W}^T(\Sigma-\hat{\Sigma})\mathbf{W}\right)\Pi\beta_* \leq \frac{1}{s}\|\beta_*\|^2 \|\Pi\|^2 \|\Sigma-\hat{\Sigma}\|\|\mathbf{W}^T\mathbf{W}\|, \nonumber\\
&=& \frac{\lambda_{W}}{s}\|\Pi\|^2 \|\Sigma-\hat{\Sigma}\|\|\beta_*\|^2. \nonumber
\end{IEEEeqnarray}
\end{proof}
The following lemma provides an upper bound for $\|\Sigma-\hat{\Sigma}\|$.
\begin{lma}\label{lma:cov_sample}
Let $\Sigma$ and $\hat{\Sigma}$ denote the covariance operator and sample covariance operator respectively, then for $\delta \in (0,1)$, we have with probability at least $1- \delta$, \[\|\Sigma - \hat{\Sigma}\| \leq c\|\Sigma\|\sqrt{\log(\frac{14r(\Sigma)}{\delta})/n},\] where $r(\Sigma) = \frac{\textnormal{Tr}(\Sigma)}{\|\Sigma\|}$ and $c_0$ is some universal constant.
\end{lma}
\begin{proof}
For $\Sigma- \hat{\Sigma}$, we first notice that\footnote{As $p \rightarrow \infty$, both $\Sigma$ and $\hat{\Sigma}$ are operators. Hence, we have misused the notation $A^TB$ to represent inner product between operator $A$ and $B$. But our analysis is not affected by this notation.}
\begin{IEEEeqnarray}{rCl}
\Sigma- \hat{\Sigma} &= & \Sigma - \frac{1}{n} D^{1/2}V(X)^TV(X)D^{1/2} = \Sigma - \frac{1}{n}\sum_{i=1}^n D^{1/2}V_{x_i}V_{x_i}^TD^{1/2}, \nonumber\\
&:= & \Sigma - \frac{1}{n}\sum_{i=1}^n \phi(x_i)\phi(x_i)^T = \sum_{i=1}^n \frac{1}{n}(\Sigma - \phi(x_i)\phi(x_i)^T),\nonumber\\
& :=& \sum_{i=1}^n R_i.\nonumber
\end{IEEEeqnarray}
Now we trivially have $\mathbb{E}(R_i) = 0$. In addition,
\begin{IEEEeqnarray}{rCl}
R_i &\preccurlyeq & \frac{1}{n}\Sigma \preccurlyeq \frac{\|\Sigma\|}{n} I = \frac{\|\Sigma\|}{n} I, \nonumber \\
R_i &\succcurlyeq & -\frac{1}{n} \phi(x_i)\phi(x_i)^T \succcurlyeq - \frac{1}{n} \|\phi(x_i)\|^2 I \succcurlyeq -\frac{\|\Sigma\|}{n}I. \nonumber
\end{IEEEeqnarray}
As a result, we have $\|R_i\| \leq \frac{\lambda_1}{n}$.
\begin{IEEEeqnarray}{rCl}
\mathbb{E}(R_i^2) & = & \frac{1}{n^2}\mathbb{E}\left(\Sigma - \phi(x_i)\phi(x_i)^T\right)^2 = \frac{1}{n^2} \left\{ \mathbb{E}\left(\phi(x_i)\phi(x_i)^T\right)^2 - 2\Sigma\phi(x_i)\phi(x_i)^T + \Sigma^2 \right\},\nonumber\\
& =& \frac{1}{n^2}\left\{ \mathbb{E}(\phi(x_i)\phi(x_i)^T)^2 - \Sigma^2 \right\} \preceq  \frac{1}{n^2} \mathbb{E}\left(\phi(x_i)\phi(x_i)^T\right)^2,\nonumber \\
&=& \frac{1}{n^2} \mathbb{E}\left(\phi(x_i)\phi(x_i)^T\phi(x_i)\phi(x_i)^T\right)\preceq  \frac{\|\Sigma\|}{n^2} \mathbb{E}(\phi(x_i)\phi(x_i)^T),\nonumber\\
&=& \frac{\|\Sigma\|}{n^2} \Sigma \preceq  \frac{\|\Sigma\|^2}{n^2}I. \nonumber
\end{IEEEeqnarray}
Thus we have, $\sum_{i=1}^n \mathbb{E}(R_i^2) \preccurlyeq  \frac{\|\Sigma\|^2}{n} I$. Hence, we have $\|\sum_{i=1}^n \mathbb{E}(R_i^2)\| \leq \frac{\|\Sigma\|^2}{n}$. Using \cite[Operator version of Theorem 3.1]{minsker2017some}, for any $\delta \in (0,1)$, with probability greater than $1- \delta$, we have \[\|\Sigma - \hat{\Sigma}\| \leq c\|\Sigma\|\sqrt{\log(\frac{14r(\Sigma)}{\delta})/n}.\]
\end{proof}

\subsubsection*{Upper Bound of Variance:} The upper bound of the variance term $\mathbf{V}_{R}$ is a bit involving, we split it into several steps. We first note that some simple algebra yield a basic upper bound:
\begin{eqnarray}
\mathbf{V}_{R} &=&  \int_{\mathcal{X}} \mathbb{E}_{\pmb{\epsilon}}\left\{ \mathbf{z}_x(\mathbf{W})^T(\mathbf{Z}^T\mathbf{Z})^{\dagger}\mathbf{Z}^T\pmb{\epsilon}\pmb{\epsilon}^T\mathbf{Z}(\mathbf{Z}^T\mathbf{Z})^{\dagger}\mathbf{z}_x(\mathbf{W})\right\},\nonumber \\
& \leq & \sigma^2 \int_{\mathcal{X}}  \mathbf{z}_x(\mathbf{W})^T(\mathbf{Z}^T\mathbf{Z})^{\dagger}\mathbf{Z}^T\mathbf{Z}(\mathbf{Z}^T\mathbf{Z})^{\dagger}\mathbf{z}_x(\mathbf{W}),\nonumber\\
&=& \sigma^2 \mathbb{E}_x\left(\mathbf{z}_x(\mathbf{W})^T(\mathbf{Z}^T\mathbf{Z})^{\dagger}\mathbf{z}_x(\mathbf{W})\right),\nonumber\\
& =&\sigma^2 \mathbb{E}_x\text{Tr}(\mathbf{z}_x(\mathbf{W})\mathbf{z}_x(\mathbf{W})^T(\mathbf{Z}^T\mathbf{Z})^{\dagger}),\nonumber\\
&=& \sigma^2\frac{1}{s}\text{Tr}\left(\mathbf{W}^T\mathbb{E}_x(D^{1/2} V_x V_x^TD^{1/2})\mathbf{W}(\mathbf{Z}^T\mathbf{Z})^{\dagger}\right), \nonumber\\
& =&\sigma^2 \frac{1}{s}\text{Tr}\left(\mathbf{W}^T\Sigma \mathbf{W}(\mathbf{Z}^T\mathbf{Z})^{\dagger}\right),\nonumber\\
&=& \sigma^2\frac{1}{s}\text{Tr}\left(\mathbf{W}^T\Sigma \mathbf{W}(\mathbf{W}^T\frac{1}{s}D^{1/2}V(X)V(X)D^{1/2} \mathbf{W})^{\dagger}\right), \nonumber\\
& =& \frac{\sigma^2}{n}\text{Tr}\left(\mathbf{W}^T\Sigma \mathbf{W}(\mathbf{W}^T\frac{1}{n}D^{1/2}V(X)V(X)D^{1/2} \mathbf{W})^{\dagger}\right), \nonumber\\
& =& \frac{\sigma^2}{n}\text{Tr}\left(\mathbf{W}^T\Sigma \mathbf{W} (\mathbf{W}^T\hat{\Sigma} \mathbf{W})^{\dagger})\right). \label{var_trace}
\end{eqnarray}
As point out before, $\Sigma$ and $\hat{\Sigma}$ are both positive semidefinite, so they admit eigendecompostion, denoted as $\Sigma = VDV^T$ and $\hat{\Sigma}= \hat{V}\hat{D}\hat{V}^T$. We now denote $\mathbf{w}_i \in \mathbb{R}^s$ to be the $i$-th column of $\mathbf{W}^T$. Since standard normal vector is invariant under orthonormal transformation,
\begin{eqnarray}
Eq.(\ref{var_trace}) &= & \frac{\sigma^2}{n}\text{Tr}\left(\mathbf{W}^T VD V^T\mathbf{W} (\mathbf{W}^T\hat{V}\hat{D}\hat{V}^T \mathbf{W})^{\dagger})\right),\nonumber\\
& =& \frac{\sigma^2}{n}\text{Tr}\left(\mathbf{W}^T D \mathbf{W} (\mathbf{W}^T\hat{D} \mathbf{W})^{\dagger})\right),\nonumber \\
& =& \frac{\sigma^2}{n} \sum_{i=1}^p \lambda_i \mathbf{w}_i^T \left(\mathbf{W}^T\hat{D} \mathbf{W}\right)^{\dagger} \mathbf{w}_i, \nonumber \\
& =& \frac{\sigma^2}{n} \sum_{i=1}^p \lambda_i \mathbf{w}_i^T \left(\sum_{i=1}^n \hat{\lambda}_i \mathbf{w}_i \mathbf{w}_i^T\right)^{\dagger} \mathbf{w}_i,  \label{trace_decom}
\end{eqnarray}
note last equality is because we are in the overparametrized regime where $s \geq n$, hence $\hat{\Sigma}$ has $n$ non-zero eigenvalues. To control the variance term, we need to study Eq.(\ref{trace_decom}). To this end, we define the following terms:
\begin{IEEEeqnarray}{rCl}
\hat{A} = \sum_{i=1}^n \hat{\lambda}_i \mathbf{w}_i \mathbf{w}_i^T, ~~~~ \hat{A}_k = \sum_{i>k}^n \hat{\lambda}_i \mathbf{w}_i \mathbf{w}_i^T \nonumber
\end{IEEEeqnarray}
$\hat{A}$ is a sum of $n$ rank one operator, so it has at most $N$ non-negative eigenvalues. We let $\mu_{1}(\hat{A}) \geq \dots \geq \mu_{n}(\hat{A})$ to be its eigenvalues. We now study the properties of these eigenvalues (the proof follows closely from \cite[Lemma 4]{bartlett2020benign}). 

\begin{lma}\label{lma:eigen_A_bd}
Let $\hat{A} = \sum_{i=1}^n \hat{\lambda}_i \mathbf{w}_i \mathbf{w}_i^T$, where $\mathbf{w}_i \in \mathbb{R}^{s}$ be a random vector with each entry being i.i.d, unit variance and $\sigma_w$-subgaussian random variables. There is a universal constant $b_1$ such that with probability at least $1-2e^{-t}$, we have 
\begin{eqnarray}
\sum_{i=1}^{n} \hat{\lambda}_{i}-\Lambda \leq \mu_{n}(\hat{A}) \leq \mu_{1}(\hat{A}) \leq \sum_{i=1}^{n} \hat{\lambda}_{i}+\Lambda, \label{eigen_concen_1}
\end{eqnarray}
where
\[\Lambda=b_1\left(\hat{\lambda}_{1}(t+n \log 9)+\sqrt{(t+n \log 9) \sum_{i=1}^{n} \hat{\lambda}_{i}^{2}}\right).\]
Further, there is a universal constant $b_2$ such that with probability at least $1-2e^{-n/b_2}$
\begin{eqnarray}
\frac{1}{b_2}\sum_{i=1}^{n} \hat{\lambda}_{i}-b_2 \hat{\lambda}_{1} n \leq \mu_{n}(\hat{A}) \leq \mu_{1}(\hat{A}) \leq b_2\sum_{i=1}^{n} \hat{\lambda}_{i}+ b_2\hat{\lambda}_{1} n.  \label{eigen_concen_2}
\end{eqnarray}
In addition, with the same probability bound, we have 
\begin{eqnarray}
\frac{1}{b_2}\sum_{i> k}^{n} \hat{\lambda}_{i}-b_2 \hat{\lambda}_{k+1} n \leq \mu_{n}(\hat{A}_k) \leq \mu_{1}(\hat{A}_k) \leq b_2\sum_{i>k}^{n} \hat{\lambda}_{i}+ b_2\hat{\lambda}_{k+1}n.  \label{eigen_concen_3}
\end{eqnarray}
\end{lma}
\begin{proof}
For any unit vector $\mathbf{v} \in \mathbb{R}^s$, we have $\mathbf{v}^T\mathbf{w}_i$ is still $\sigma_w^2$-subgaussian, this implies that $\mathbf{v}^T\mathbf{u}_i\mathbf{w}_i^T\mathbf{v}-1 = (\mathbf{v}^T\mathbf{w}_i)^2-1$ is centered and $\sigma_w^2$ subexponential. Applying Lemma \ref{lma: sub_exp_sum}, we have for any unit vector $\mathbf{v}$, there is a universal constant $c_1$, with probability at least $1-2 e^{-t}$, \[\left|\mathbf{v}^T\hat{A}\mathbf{v} - \sum_{i=1}^n \hat{\lambda}_i\right| \leq c_1 \left(\hat{\lambda}_1 t +\sqrt{t \sum_{i=1}^n \hat{\lambda}_{i}^{2}}\right).\] Now since $\hat{A}$ has at most $n$ non-negative eigenvalues, we let the $n$ dimensional subspace spanned by $\hat{A}$ as $\mathcal{A}^n$, let $\mathcal{N}_{\omega}$ to be the $\omega$-net of $\mathcal{S}^{n-1}$ with respect to the Euclidean distance, where $\mathcal{S}^{n-1}$ is the unit sphere in $\mathcal{A}^n$. We let $\omega = \frac{1}{4}$, implying that $|\mathcal{N}_{\omega}| \leq 9^n$. Apply union bound, for every $\mathbf{v} \in \mathcal{N}_{\epsilon}$, we have with probability at least $1-2e^{-t}$ \[\left|\mathbf{v}^T\hat{A}\mathbf{v} - \sum_{i=1}^n \hat{\lambda}_i\right| \leq c_1\left( \hat{\lambda}_{1}(t+n \log 9)+\sqrt{(t+n \log 9) \sum_{i=1}^{n} \hat{\lambda}_{i}^{2}}\right).\] 
By Lemma \ref{lma:e-net}, since $\omega = \frac{1}{4}$, for any $\mathbf{v} \in \mathcal{S}^{n-1}$, we have \[\left|\mathbf{v}^T\hat{A}\mathbf{v} - \sum_{i=1}^n \hat{\lambda}_i\right| \leq c_2\left(\hat{\lambda}_{1}(t+n \log 9)+\sqrt{(t+n \log 9) \sum_{i=1}^{n} \hat{\lambda}_{i}^{2}}\right) := \Lambda.\]
Thus, with probability $1- 2e^{-t}$, we have \[\left\|\hat{A} - \sum_{i=1}^n \hat{\lambda}_i I_n \right\| \leq \Lambda.\]
We further simplify $\Lambda$ now. Notice that when $t \leq \frac{n}{c_3}$, $(t+n \log 9) \leq c_4 n$. Hence,
\begin{IEEEeqnarray}{rCl}
\Lambda & \leq & c_5\hat{\lambda}_1n + \sqrt{c_6n \hat{\lambda}_1 \sum_{i=1}^n \hat{\lambda}_i} \nonumber \\
& \leq & c_5\hat{\lambda}_1n + \frac{1}{2}c_6c_7\hat{\lambda}_1n + \frac{1}{2c_7}\sum_{i=1}^n \hat{\lambda}_i~~~ \left(\text{we~use~}\sqrt{xy} \leq \frac{x+y}{2}\right).\nonumber
\end{IEEEeqnarray}
Combining this with Eq.(\ref{eigen_concen_1}) yields Eq.(\ref{eigen_concen_2}). Using the same proof with $\hat{A}_{k}$ replacing $\hat{A}$, we obtain Eq.(\ref{eigen_concen_3}).
\end{proof}

\begin{lma}\label{lma:var_up}
For universal constants $b$, recall the definition of  $k^*$ as \[k^* = \min \left\{0 \leq j \leq n, \frac{\sum_{i > j}^N \hat{\lambda}_i }{\hat{\lambda}_{j+1}} \geq b n \right\},\] we then have with probability greater than $1- 2e^{-\frac{n}{c}}$, \[\mathbf{V}_{R} \leq b \sigma^2\frac{s}{n} \frac{\textnormal{Tr}(\Sigma)}{\sum_{i>k^*}^n\hat{\lambda}_i}.\]
\end{lma}

\begin{proof}
Since $\hat{A}-\hat{A}_{k} = \sum_{i=1}^k \hat{\lambda}_i \mathbf{w}_i\mathbf{w}_i^T$ is positive semidefinite, we have that $\mu_{n}(\hat{A}) \geq \mu_{n}(\hat{A}_k)$. By Lemma \ref{lma:eigen_A_bd}, with probability greater than $1-e^{-n/c_1}$, we can lower bound the smallest non-zero eigenvalue of $\hat{A}$ as \[\mu_{n}(\hat{A}) \geq \mu_{n}(\hat{A}_k) \geq \frac{1}{c_1} \sum_{i> k}^{n} \hat{\lambda}_{i}-c_1 \hat{\lambda}_{k+1} n.\] Assuming  there is $0 \leq j \leq n$ such that $\sum_{i > j}^n \hat{\lambda}_i \geq c \hat{\lambda}_{j+1}n$ and $c > c_1^2$, we have \[\mu_{n}(\hat{A}) \geq \frac{1}{c_1} \sum_{i>j}^n\hat{\lambda}_i - \frac{c_1}{c}\sum_{i>j}^n\hat{\lambda}_i = \frac{1}{c_1c_2}\sum_{i >j}^n\hat{\lambda}_i .\] Thus,
\begin{IEEEeqnarray}{rCl}
\mathbf{V}_{R} &\leq & Eq.(\ref{trace_decom}) \leq \frac{\sigma^2}{n} \sum_{i=1}^p \mu_{n}(\hat{A})^{-1}\lambda_i \mathbf{w}_i^T\mathbf{w}_i, \nonumber \\
& \leq &  \frac{\sigma^2}{n} \left(\frac{1}{c_1c_2}  \sum_{i>j}^n\hat{\lambda}_i\right)^{-1} \sum_{i=1}^p \lambda_i \mathbf{w}_i^T\mathbf{w}_i\nonumber 
\end{IEEEeqnarray}
since $\mathbf{w}_i$ are standard Gaussian random vector, using Lemma \ref{lma:norm_sug}, we have for some universal constants $c_3,c_4$, with probability greater than $1-e^{-s/c_3}$, $\mathbf{w}_i^T\mathbf{w}_i \leq c_4 s$. Hence, if we choose $c > c_1^2$ such that $1-e^{-n/c} \leq 1-e^{-s/c_3}$, then with probability greater than $1-2e^{-n/c}$, we have for $k^*$, \[\mathbf{V}_{R} \leq c \sigma^2\frac{s}{n} \frac{\text{Tr}(\Sigma)}{\sum_{i>k^*}^N\hat{\lambda}_i} .\] 
\end{proof}

\begin{lma}\label{lma:var_low}
There exists a universal constant $b$ such that with probability greater than $1- 2e^{-\frac{n}{b}}$, we have \[\mathbf{V}_{R} \geq b \sigma^2\frac{s}{n} \frac{\textnormal{Tr}(\Sigma)}{\sum_{i>k^*}^n\hat{\lambda}_i}.\]
\end{lma}
\begin{proof}
The proof is similar to the upper bound, where the difference is that we use the upper bound of $\mu_1(\hat{A})$ to obtain the lower bound of $\mathbf{V}_{R}$.
\end{proof}

Now equipped with the above tools, we are ready to prove Proposition \ref{theo:over_rff}.
\begin{proof}
For the upper bound, combining Lemma \ref{lma:bia_var},  \ref{lma:bia_up}, \ref{lma:cov_sample} and \ref{lma:var_up} yields the result. For the lower bound, we simply notice that $R(\tilde{\beta}) \geq \mathbf{V}_{R}$ and apply Lemma \ref{lma:var_low} to achieve the result.
\end{proof}

\subsection{Proof of Theorem \ref{theo: over_rff_noisy}} \label{sec:gau_proof}
%The proof starts again with the bias-variance decomposition from Lemma \ref{lma:noisy_bias_var}. 
We deal with the bias first in the next section.
\subsubsection{Upper Bound on Bias}
We first notice that
\begin{IEEEeqnarray}{rCl} 
\mathbf{B}_{\xi} & =  & \int_{\mathcal{X}} \left\|\mathbf{z}_x^{\xi}(\mathbf{W})^T\Pi_{\xi}\beta_*^{\xi}\right\|^2d\rho(x) = \mathbb{E}_x\left\{ \beta_*^{\xi T}  \Pi_{\xi}\mathbf{z}_x^{\xi}(\mathbf{W})\mathbf{z}_x^{\xi}(\mathbf{W})^T\Pi_{\xi}\beta_*^{\xi}\right\},\nonumber\\
& =&  \beta_*^{\xi T}  \Pi_{\xi}\left\{ \mathbb{E}_x(\mathbf{z}_x^{\xi}(\mathbf{W})\mathbf{z}_x^{\xi}(\mathbf{W})^T) - \frac{1}{n}\mathbf{Z}_{\xi}^T\mathbf{Z}_{\xi}\right\}\Pi_{\xi}\beta_*^{\xi}. \nonumber
\end{IEEEeqnarray}
By definition, 
\begin{IEEEeqnarray}{rCl} 
\frac{1}{n}\mathbf{Z}_{\xi}^T\mathbf{Z}_{\xi} &= & \frac{1}{n} (\mathbf{Z} + \Xi)^T(\mathbf{Z} + \Xi)= \frac{1}{n} ( \mathbf{Z}^T\mathbf{Z} + \mathbf{Z}^T\Xi + \Xi^T\mathbf{Z} + \Xi^T\Xi). \nonumber
\end{IEEEeqnarray}
Thus,
\begin{eqnarray}
&&\mathbb{E}_x(\mathbf{z}_x^{\xi}(\mathbf{W})\mathbf{z}_x^{\xi}(\mathbf{W})^T )- \frac{1}{n}\mathbf{Z}_{\xi}^T\mathbf{Z}_{\xi} \nonumber\\
&=& \mathbb{E}_x(\mathbf{z}_x^{\xi}(\mathbf{W})\mathbf{z}_x^{\xi}(\mathbf{W})^T ) - \frac{1}{n} \mathbf{Z}^T\mathbf{Z} \label{bias_exp_samp}\\
&& - \frac{1}{n}\mathbf{Z}^T\Xi \label{bias_Z_Xi}\\
&& -\frac{1}{n}\Xi^T\mathbf{Z} \label{bias_Xi_Z}\\
&& - \frac{1}{n}\Xi^T\Xi. \label{bias_Xi_Xi}
\end{eqnarray}

\paragraph{Upper Bound Eq.(\ref{bias_exp_samp})} By definition, we have 
\begin{IEEEeqnarray}{rCl} 
&&\mathbf{z}_x^{\xi}(\mathbf{W})\mathbf{z}_x^{\xi}(\mathbf{W})^T  =  (\mathbf{z}_x(\mathbf{W}) + \pmb{\xi}) (\mathbf{z}_x(\mathbf{W}) + \pmb{\xi})^T, \nonumber\\
& =& \mathbf{z}_x(\mathbf{W})\mathbf{z}_x(\mathbf{W})^T + \mathbf{z}_x(\mathbf{W})\pmb{\xi}^T+ \pmb{\xi}^T\mathbf{z}_x(\mathbf{W}) + \pmb{\xi}\pmb{\xi}^T. \nonumber
\end{IEEEeqnarray}
As a result, Eq.(\ref{bias_exp_samp}) can be written as
\begin{eqnarray}
&&\mathbb{E}_{x}\left\{ \mathbf{z}_x(\mathbf{W})\mathbf{z}_x(\mathbf{W})^T -\frac{1}{n} \mathbf{Z}^T\mathbf{Z}\right\} \label{bias_cov_samp}\\
&& \: + \mathbb{E}_{x} (\mathbf{z}_x(\mathbf{W})\pmb{\xi}^T ) \label{bias_z_xi} \\
&& \: + \mathbb{E}_{x} (\pmb{\xi}\mathbf{z}_x(\mathbf{W})^T ) \label{bias_xi_z}\\
&& \: + \pmb{\xi}\pmb{\xi}^T. \label{bias_xi_xi}
\end{eqnarray}
Eq.(\ref{bias_cov_samp}) $= \frac{1}{s} (\mathbf{W}^T \Sigma \mathbf{W} -\mathbf{W}^T \hat{\Sigma} \mathbf{W})$, its operator norm can be trivially upper bounded by $\frac{\lambda_W}{s}\|\Sigma - \hat{\Sigma}\|$, where we recall $\lambda_W = \|\mathbf{W}^T\mathbf{W}\|$. We now denote $\pmb{\xi} = \frac{\sigma_0}{\sqrt{s}} \mathbf{u}$, where $\mathbf{u} \in \mathbb{R}^s$ with each entry being i.i.d standard normal by definition of $\pmb{\xi}$. Recall that $\mathbf{z}_x(\mathbf{W}) = \frac{1}{\sqrt{s}}\mathbf{W}^TD^{\frac{1}{2}}V(x)$, we have \[\mathbf{z}_x(\mathbf{W})\pmb{\xi}^T = \frac{\sigma_0}{s} \mathbf{W}^TD^{\frac{1}{2}}V(x) \mathbf{u}^T, \]
where $\mathbf{w}^{(i)} \in \mathbb{R}^p$ as the $i$-th column of $\mathbf{W}$.It is easy to see \[\mathbf{w}^{(i)T}D^{\frac{1}{2}}V(x) = \sum_j^p \sqrt{\lambda}_je_j(x)\mathbf{w}_j^{(i)} \in \mathbb{R}.\] is a normal random variable. Immediately, we can see that it has mean $0$ and variance $\text{Var}(\mathbf{w}^{(i)T}D^{\frac{1}{2}}V_x) = \sum_j^p\lambda_j e_j(x)e_j(x) = k(x,x)$. As a result, we have $\mathbf{w}^{(i)T}D^{\frac{1}{2}}V(x) \sim \mathcal{N}(0,k(x,x))$.
Hence, \[\mathbf{W}^TD^{\frac{1}{2}}V(x) = \left[\mathbf{w}^{(1)T}D^{\frac{1}{2}}V(x),\dots,\mathbf{w}^{(s)T}D^{\frac{1}{2}}V(x)\right]^T = \sqrt{k(x,x)}\mathbf{v},\] where $\mathbf{v} \in \mathbb{R}^s$ is a Gaussian random vector with each entry being i.i.d $\sim \mathcal{N}(0,1)$. As a result,
\begin{IEEEeqnarray}{rCl} 
\mathbf{z}_x(\mathbf{W})\pmb{\xi}^T & =&  \frac{\sigma_0}{s} \mathbf{W}^TD^{\frac{1}{2}}V(x) \mathbf{u}^T = \frac{\sigma_0}{s} \sqrt{k(x,x)} \mathbf{v} \mathbf{u}^T, \nonumber
\end{IEEEeqnarray}
Hence, by Lemma \ref{lma:con_norm_pro}, for some universal constant $c_1,c_2,c_3$, if we let $t \leq s/c_1$, then with probability greater than $1- 2e^{-s/c_1}$, we can now upper bound $\mathbb{E}_x(\mathbf{z}_x(\mathbf{W})\pmb{\xi}^T)$ as:
\begin{eqnarray}
\left\|\mathbb{E}_x(\mathbf{z}_x(\mathbf{W})\pmb{\xi}^T)\right\| &=& \left\|\mathbb{E}_x\left(\frac{\sigma_0}{s} \sqrt{k(x,x)} \mathbf{v} \mathbf{u}^T\right) \right\|,  \nonumber \\
& \leq & \frac{\sigma_0}{s} \left\|\sqrt{\mathbb{E}_x(k(x,x))}\right\| \|\mathbf{v}\mathbf{u}^T\|, \nonumber\\
& \leq &  \frac{\sigma_0}{s} \sqrt{C_0} \left\|\sum_{i=1}^s v_iu_i\right\|, \nonumber\\
& \leq & \frac{\sigma_0}{s} \sqrt{C_0} c_2s = c_3\sigma_0. \label{bias_z_xi_up}
\end{eqnarray}
Thus, Eq.(\ref{bias_z_xi}) $\leq c_3 \sigma_0 I$. Eq.(\ref{bias_xi_z}) has the same upper bound as Eq.(\ref{bias_z_xi}) since they have exactly the same eigenvalues.\\

For Eq.(\ref{bias_xi_xi}), we apply Lemma \ref{lma:norm_sug}, $\|\pmb{\xi}\pmb{\xi}^T\| = \frac{\sigma_0^2}{s}\mathbf{u}^T\mathbf{u}$, we have $ \|\pmb{\xi}\pmb{\xi}^T\| \leq c_3 \sigma_0^2$ with probability greater than $1- e^{-s/c_4}$. Combining this all together, we have with probability greater than $1- 3e^{-s/c_5} \geq 1- 3e^{-n/c_5}$, $c_5 = \max\{c_1,c_4\}$,
\begin{eqnarray}
\|\textnormal{Eq}.(\ref{bias_exp_samp})\| & \leq & \|\textnormal{Eq}.(\ref{bias_cov_samp})\| + \|\textnormal{Eq}.(\ref{bias_xi_z})\| + \|\textnormal{Eq}.(\ref{bias_z_xi})\| + \|\textnormal{Eq}.(\ref{bias_xi_xi})\|\nonumber\\
&\leq & \frac{\lambda_W}{s}\|\Sigma -\hat{\Sigma}\| + 2c_2\sigma_0 + c_3 \sigma_0^2. \label{bias_exp_samp_up}
\end{eqnarray}

\paragraph{Upper Bound Eq.(\ref{bias_Z_Xi}) \& Eq.(\ref{bias_Xi_Z})} Recall that $\mathbf{Z} = \frac{1}{\sqrt{s}}V(X)D^{\frac{1}{2}}\mathbf{W}$, and $\hat{\Sigma} = \frac{1}{n}D^{\frac{1}{2}}V(X)^TV(X)D^{\frac{1}{2}}$ has eigenvalues $\hat{\lambda}_1,\cdots, \hat{\lambda}_n \geq 0$. Hence, $\frac{1}{\sqrt{n}}D^{\frac{1}{2}}V(X)^T$ has singular values as $\{\hat{\sigma}_i\}_{i=1}^n$ with $\hat{\sigma}_i^2 = \hat{\lambda}_i$. And for some orthonormal basis $\hat{\mathcal{U}}_{X} \in \mathbb{R}^{p\times p}$ and $\hat{\mathcal{V}}_{X} \in \mathbb{R}^{n\times n}$, $\frac{1}{\sqrt{n}}D^{\frac{1}{2}}V(X)^T$ must admit singular value decomposition as \[\frac{1}{\sqrt{n}}D^{\frac{1}{2}}V(X)^T = \hat{\mathcal{U}}_{X} \hat{\mathcal{D}}_{X}\hat{\mathcal{V}}_X,\] where $\hat{\mathcal{D}}_{X}\in \mathbb{R}^{p \times n}$ is the matrix with diagonal elements being the singular values and $0$ otherwise. 
Therefore:
\begin{IEEEeqnarray}{rCl} 
\frac{1}{n}\mathbf{Z}^T\Xi &=& \frac{1}{n\sqrt{s}}\mathbf{W}^TD^{\frac{1}{2}}V(X)^T\Xi= \frac{1}{\sqrt{ns}}\mathbf{W}^T \hat{\mathcal{U}}_{X}\hat{\mathcal{D}}_{X}\hat{\mathcal{V}}_X\Xi, \nonumber\\
&=& \frac{1}{\sqrt{ns}}\mathbf{W}^T\hat{\mathcal{D}}_{X}\mathbf{U}\frac{\sigma_0}{\sqrt{s}} = \frac{\sigma_0}{s\sqrt{n}} \sum_i^n \hat{\sigma}_i \mathbf{w}_i \mathbf{u}_i^T. \nonumber
\end{IEEEeqnarray}
where $\mathbf{U} = [\mathbf{u}_1,\dots, \mathbf{u}_n]^T \in \mathbb{R}^{n \times s}$ with each entry being i.i.d $\sim\mathcal{N}(0,1)$. 

We would like to apply Lemma \ref{lma:eigen_norm_pro} to the above equation. But we need to consider the value of $\{\hat{\sigma}_i\}_{i=1}^n$ as $\hat{\sigma}_i$ can either be $\sqrt{\hat{\lambda}_i}$ or $-\sqrt{\hat{\lambda}_i}$. We can see that $\mathbf{w}_i$ is a standard Gaussian random vector, $-\mathbf{w}_i$ and $\mathbf{w}_i$ have the same distribution. As a result, without loss of generality, we may assume that $\hat{\sigma}_i = \sqrt{\hat{\lambda}_i} $, so $\{\hat{\sigma}_i\}_{i=1}^n$ is non-negative and non-increasing. Now apply Lemma \ref{lma:eigen_norm_pro} and notice that $\sum_{i=1}^n \hat{\sigma}_i^2 =\sum_i \hat{\lambda}_i$ is finite by our assumption, we conclude with probability greater than $1-e^{-n/c_6}$,
\begin{eqnarray}
\left\|\frac{1}{n}\Xi^T\mathbf{Z}\right\| =\left\|\frac{1}{n}\mathbf{Z}^T\Xi\right\| \leq \frac{\sigma_0}{s\sqrt{n}}c_7 n\sqrt{\sum_i \hat{\lambda}_i}\leq c_8 \frac{\sigma_0}{\sqrt{s}}. \label{bias_Z_Xi_up}
\end{eqnarray}

\paragraph{Upper Bound Eq.(\ref{bias_Xi_Xi})} For this, we simply notice that
\begin{IEEEeqnarray}{rCl} 
\frac{1}{n}\Xi^T\Xi &=& \frac{\sigma_0^2}{sn} \mathbf{U}^T\mathbf{U} = \frac{\sigma_0^2}{sn} \sum_i^n \mathbf{u}_i\mathbf{u}_i^T. \nonumber 
\end{IEEEeqnarray}
Note that since we are in the overparametrized regime ($s \geq n$), $\mathbf{U}^T\mathbf{U}$ has at most $n$ non-zero eigenvalues.
Therefore by Lemma \ref{lma:eigen_A_bd}, Eq.(\ref{eigen_concen_2}), with probability greater than $1-e^{-n/c_9}$
\begin{eqnarray}
\left\|\frac{1}{n}\Xi^T\Xi \right\| & =& \left\|\frac{\sigma_0^2}{sn} \sum_i^n \mathbf{u}_i\mathbf{u}_i^T\right\|, \nonumber\\
& \leq & \frac{\sigma_0^2}{sn} \left\|\sum_i^n \mathbf{u}_i\mathbf{u}_i^T\right\|,\nonumber\\
& \leq & \frac{\sigma_0^2}{sn} \left(\frac{1}{c_{10}}+c_{10} n\right) ~~~~\textnormal{by~Eq.(\ref{eigen_concen_2})},\nonumber\\
& =& c_{11}\frac{\sigma_0^2n}{sn} \leq c_{11} \frac{\sigma_0^2}{s}. \label{bias_Xi_Xi_up}
\end{eqnarray}
Now combine Eq.(\ref{bias_exp_samp_up}), (\ref{bias_Z_Xi_up}) and (\ref{bias_Xi_Xi_up}) and Lemma \ref{lma:cov_sample}, taking the universal constants $c = \max\{c_1,\dots,c_{11}\}$, with probability greater than $1-\delta- 6e^{-n/c}$
\begin{IEEEeqnarray}{rCl} 
\mathbf{B}_{\xi} & \leq & c \left(\frac{\lambda_W}{s} \|\Sigma\|\sqrt{\log(\frac{14r(\Sigma)}{\delta})/n} + \sigma_0 + \sigma_0^2\right) \|\Pi_{\xi}\|^2\|\beta_{*}^{\xi}\|^2, \nonumber
\end{IEEEeqnarray}
note that we have omitted the $\frac{\sigma_0}{\sqrt{s}}$ and $\frac{\sigma_0^2}{s}$ terms because they are dominated by $\sigma_0$ and $\sigma_0^2$ respectively.

\subsubsection{Upper Bound on Variance} \label{sec: upper_var_noisy}
For the variance, we have 
\begin{eqnarray}
\mathbf{V}_{\xi} &= & \int_{\mathcal{X}} \mathbb{E}_{\pmb{\epsilon}}\left\| \mathbf{z}_x^{\xi}(\mathbf{W})^T\left(\mathbf{Z}_{\xi}^T\mathbf{Z}_{\xi}\right)^{\dagger}\mathbf{Z}_{\xi}^T\pmb{\epsilon}\right\|^2 d\rho(x), \nonumber \\
& = & \int_{\mathcal{X}} \mathbb{E}_{\pmb{\epsilon}}\left\{ \mathbf{z}_x^{\xi}(\mathbf{W})^T\left(\mathbf{Z}_{\xi}^T\mathbf{Z}_{\xi}\right)^{\dagger}\mathbf{Z}_{\xi}^T\pmb{\epsilon}\pmb{\epsilon}^T \mathbf{Z}_{\xi}\left(\mathbf{Z}_{\xi}^T\mathbf{Z}_{\xi}\right)^{\dagger}\mathbf{z}_x^{\xi}(\mathbf{W}) \right\}d\rho(x), \nonumber\\
& \leq & \sigma^2 \int_{\mathcal{X}} \mathbf{z}_x^{\xi}\left(\mathbf{W})^T(\mathbf{Z}_{\xi}^T\mathbf{Z}_{\xi}\right)^{\dagger}\mathbf{Z}_{\xi}^T\mathbf{Z}_{\xi}\left(\mathbf{Z}_{\xi}^T\mathbf{Z}_{\xi}\right)^{\dagger}\mathbf{z}_x^{\xi}(\mathbf{W})d\rho(x), \nonumber\\
& =& \sigma^2 \mathbb{E}_x\left\{ \mathbf{z}_x^{\xi}(\mathbf{W})^T\left(\mathbf{Z}_{\xi}^T\mathbf{Z}_{\xi}\right)^{\dagger}\mathbf{z}_x^{\xi}(\mathbf{W})\right\}, \nonumber\\
& =& \sigma^2 \mathbb{E}_x \left\{ \textnormal{Tr} \left[\mathbf{z}_x^{\xi}(\mathbf{W}) \mathbf{z}_x^{\xi}(\mathbf{W})^T\left(\mathbf{Z}_{\xi}^T\mathbf{Z}_{\xi}\right)^{\dagger}\right]\right\}, \nonumber\\
& =& \frac{\sigma^2}{n} \textnormal{Tr}\left\{\mathbb{E}_x\left[\mathbf{z}_x^{\xi}(\mathbf{W}) \mathbf{z}_x^{\xi}(\mathbf{W})^T\right] \left(\frac{1}{n}\mathbf{Z}_{\xi}^T\mathbf{Z}_{\xi}\right)^{\dagger} \right\}, \nonumber \\
& =&  \frac{s\sigma^2}{n} \textnormal{Tr} \left\{\mathbb{E}_x\left[\mathbf{z}_x(\mathbf{W}) \mathbf{z}_x(\mathbf{W})^T\right] \left(\frac{s}{n}\mathbf{Z}_{\xi}^T\mathbf{Z}_{\xi}\right)^{\dagger} \right\}, \label{var_z_z} \\
&& + \frac{s\sigma^2}{n} \textnormal{Tr} \left\{\mathbb{E}_x\left[\mathbf{z}_x(\mathbf{W}) \pmb{\xi}^T\right] \left(\frac{s}{n}\mathbf{Z}_{\xi}^T\mathbf{Z}_{\xi}\right)^{\dagger} \right\}, \label{var_z_xi} \\
&& + \frac{s\sigma^2}{n} \textnormal{Tr} \left\{\mathbb{E}_x\left[\pmb{\xi}\mathbf{z}_x(\mathbf{W})^T\right] \left(\frac{s}{n}\mathbf{Z}_{\xi}^T\mathbf{Z}_{\xi}\right)^{\dagger} \right\}, \label{var_xi_z} \\
&& + \frac{s\sigma^2}{n} \textnormal{Tr} \left\{\pmb{\xi}\pmb{\xi}^T \left(\frac{s}{n}\mathbf{Z}_{\xi}^T\mathbf{Z}_{\xi}\right)^{\dagger} \right\}. \label{var_xi_xi} 
\end{eqnarray} 

To obtain the upper bound on the variance, we need to study the properties of the eigenvalues from $\frac{1}{n}\mathbf{Z}_{\xi}^T\mathbf{Z}_{\xi}$. In order to do that, we first provide a more refined way of studying the eigenvalues of $\hat{A} = \sum_{i=1}^N \hat{\lambda}_i\mathbf{w}_i\mathbf{w}_i^T$ than Lemma \ref{lma:eigen_A_bd}.
\begin{lma}\label{lma:eigen_A_bd_beta}
Let $A = \sum_{i=1}^n \lambda_i \mathbf{w}_i\mathbf{w}_i^T$, where $\mathbf{w}_i \in \mathbb{R}^s$ is a random vector with each entry being i.i.d $\sigma_w^2$-subgaussian random variable and $\{\lambda_i\}_{i=1}$ is a sequence of non-negative, non-increasing numbers with finite sum $\sum_{i=1}^n \lambda_i < \infty$. Assume $s>n$, with probability at least $1-2e^{-n/b}$,
\begin{IEEEeqnarray}{rCl}
(1- \frac{1}{bb_1} \sigma_w^2)\sum_{i=1}^n \lambda_i - \frac{b_2}{\sqrt{b}}\sigma_w^2\lambda_1n \leq \mu_{n}(A) \leq \mu_1(A) \leq (1+\frac{1}{bb_1} \sigma_w^2)\sum_{i=1}^n \lambda_i + \frac{b_2}{\sqrt{b}}\sigma_w^2\lambda_1n, \nonumber
\end{IEEEeqnarray}
where $b,b_1, b_2 > 1$ are some universal constants. In addition, with the same probability bound, we have 
\begin{IEEEeqnarray}{rCl}
(1- \frac{1}{bb_1} \sigma_w^2)\sum_{i>k}^n \lambda_i - \frac{b_2}{\sqrt{b}}\sigma_w^2\lambda_{k+1}n \leq \mu_{n}(A_k) \leq \mu_1(A_k) \leq (1+\frac{1}{bb_1} \sigma_w^2)\sum_{i>k}^n \lambda_i + \frac{b_2}{\sqrt{b}}\sigma_w^2\lambda_{k+1}n, \nonumber
\end{IEEEeqnarray}
\end{lma}
\begin{proof}
For any unit vector $\mathbf{v} \in \mathbb{R}^s$, we have $\mathbf{v}^T\mathbf{w}_i$ is $c_1\sigma_w^2$-subgaussian random variable. Hence for any $\mathbf{v}$, $\mathbf{v}^TA\mathbf{v} = \sum_{i=1}^n \lambda_i (\mathbf{v}^T\mathbf{w}_i)^2$. Applying Lemma \ref{lma: sub_exp_sum}, for any unit vector $\mathbf{v}$, there is a universal constant $c_1$, a constant $c_2$ and $t > 0$, with probability at least $1-2e^{-t/c_2}$, 
\begin{IEEEeqnarray}{rCl}
|\mathbf{v}^TA\mathbf{v} - \sum_{i=1} \lambda_i| &\leq & c_1\sigma_w^2\max \left( \lambda_1 \frac{t}{c_2}, \sqrt{\frac{t}{c_2}\sum_{i=1} \lambda_i^{2}}\right) \leq \frac{c_1}{\sqrt{c_2}}\sigma_w^2\left(\lambda_1 t + \sqrt{t\sum_{i=1} \lambda_i^2}\right). \nonumber
\end{IEEEeqnarray}
Now since $A$ has at most $n$ non-negative eigenvalues, we let the $n$ dimensional subspace spanned by $A$ as $\mathcal{A}^n$, let $\mathcal{N}_{\omega}$ to be the $\omega$-net of $\mathcal{S}^{n-1}$ with respect to the Euclidean distance, where $\mathcal{S}^{n-1}$ is the unit sphere in $\mathcal{A}^n$. We let $\omega = \frac{1}{4}$, implying that $|\mathcal{N}_{\omega}| \leq 9^n$. Apply union bound, for every $\mathbf{v} \in \mathcal{N}_{\epsilon}$, we have with probability at least $1-2 e^{-t/c_2}$ \[\left|\mathbf{v}^TA\mathbf{v} - \sum_{i=1}^n \lambda_i\right| \leq \frac{c_1}{\sqrt{c_2}}\sigma_w^2\left(\lambda_1 (t + n\log 9) + \sqrt{(t+ n \log 9)\sum_{i=1} \lambda_i^2}\right).\] 
By Lemma \ref{lma:e-net}, since $\omega = \frac{1}{4}$, for any $\mathbf{v} \in \mathcal{S}^{n-1}$, we have 
\begin{IEEEeqnarray}{rCl}
\left|\mathbf{v}^TA\mathbf{v} - \sum_{i=1}^n \lambda_i\right| &\leq & \frac{32}{9}\frac{c_1}{\sqrt{c_2}}\sigma_w^2\left(\lambda_1 (t + n\log 9) + \sqrt{(t+ n \log 9)\sum_{i=1} \lambda_i^2}\right),\nonumber\\
& = & \frac{c_3}{\sqrt{c_2}}\sigma_w^2\left(\lambda_1 (t + n\log 9) + \sqrt{(t+ n \log 9)\sum_{i=1} \lambda_i^2}\right):= \Lambda.\nonumber
\end{IEEEeqnarray}
Thus, with probability $1- 2e^{-t/c_2}$, we have \[\left\|A - \sum_{i=1} \lambda_i I_n \right\| \leq \Lambda.\]
We further simplify $\Lambda$ now. Notice that when $t \leq n$, $(t + n \log 9) \leq c_4 n$. Hence,
\begin{IEEEeqnarray}{rCl}
\Lambda & \leq &\frac{c_3}{\sqrt{c_2}}\sigma_w^2 \left(c_4\lambda_1 n + \sqrt{c_4n \sum_{i=1}^n \lambda_i^2} \right), \nonumber \\
& \leq & \frac{c_3}{\sqrt{c_2}}\sigma_w^2 \left(c_4\lambda_1 n + \sqrt{c_4n\lambda_1 \sum_{i=1}^n \lambda_i} \right),  \nonumber \\
& \leq & \frac{c_3}{\sqrt{c_2}}\sigma_w^2 \left(c_4\lambda_1 n + \frac{c_4c_6n\lambda_1}{2} + \frac{1}{2c_6}\sum_{i=1}^n \lambda_i \right) ~~~ (\text{we~use~}\sqrt{xy} \leq \frac{x+y}{2}),\nonumber\\
& = & \frac{1}{\sqrt{c_2}} \sigma_w^2(c_7 \lambda_1 n + \frac{1}{c_8} \sum_{i=1}^n \lambda_i). \nonumber
\end{IEEEeqnarray}
Therefore, with probability $1- 2 e^{-n/c_2}$, we have:
\begin{IEEEeqnarray}{rCl}
(1- \frac{1}{c_8\sqrt{c_2}}\sigma_w^2)\sum_{i=1}^n \lambda_i - \frac{c_7}{\sqrt{c_2}}\sigma_w^2\lambda_1n \leq \mu_{n}(A) \leq \mu_1(A) \leq (1+ \frac{1}{c_8\sqrt{c_2}}\sigma_w^2)\sum_{i=1}^n \lambda_i + \frac{c_7}{\sqrt{c_2}}\sigma_w^2\lambda_1n.  \nonumber
\end{IEEEeqnarray}
Using the same proof with $A_{k}$ replacing $A$, we obtain the bound for $\mu_n(A_k)$ and $\mu_1(A_k)$.
\end{proof}

We first investigate the property of $\frac{1}{\sqrt{n}}(\mathbf{Z}+\pmb{\Xi})$. As discussed before, for some orthonormal basis $\hat{\mathcal{U}}_{X} \in \mathbb{R}^{p\times p}$ and $\hat{\mathcal{V}}_{X} \in \mathbb{R}^{n\times n}$, $\frac{1}{\sqrt{n}}D^{\frac{1}{2}}V(X)^T$ must admit singular value decomposition as \[\frac{1}{\sqrt{n}}D^{\frac{1}{2}}V(X)^T = \hat{\mathcal{U}}_{X} \hat{\mathcal{D}}_{X}\hat{\mathcal{V}}_X,\] where $\hat{\mathcal{D}}_{X}\in \mathbb{R}^{p \times n}$ is the matrix with diagonal elements being the singular values $\{\hat{\sigma}_{i}\}_{i=1}^n$ and $0$ otherwise. As a result, we can write \[\frac{1}{\sqrt{n}}\mathbf{Z} =\frac{1}{\sqrt{n}}V(X)D^{\frac{1}{2}}\mathbf{W} = \hat{\mathcal{V}}_X\hat{\mathcal{D}}_{X}^T\hat{\mathcal{U}}_{X}\mathbf{W}.\]
Since Gaussian random variable is invariant under orthogonal transformation, we can further simplify the above equation as:
\[\frac{1}{\sqrt{n}}\mathbf{Z} = \hat{\mathcal{V}}_X\hat{\mathcal{D}}_{X}^T\mathbf{W}.\]
In addition, the singular value matrix $\hat{\mathcal{D}}_X \in \mathbb{R}^{p\times n}$ has only $n$ diagonal entry being non-zero, if we let $\mathcal{D}_X = \text{diag}[\hat{\sigma}_1,\dots,\hat{\sigma}_n] \in \mathbb{R}^{n\times n}$, then equivalently, we can write \[\frac{1}{\sqrt{n}}\mathbf{Z} = \hat{\mathcal{V}}_X\mathcal{D}_{X}\mathbf{W}.\]
Note that the $\mathbf{W}$ in the last equation represents a $n\times s$ Gaussian random matrix.

Therefore,
\begin{IEEEeqnarray}{rCl} 
\frac{1}{\sqrt{n}}(\mathbf{Z}+\pmb{\Xi}) &=& \frac{1}{\sqrt{s}}\left(\hat{\mathcal{V}}_X\mathcal{D}_{X}\mathbf{W} + \frac{\sigma_0}{\sqrt{n}}\mathbf{U}\right) 
= \frac{1}{\sqrt{s}}\hat{\mathcal{V}}_X\left(\mathcal{D}_X\mathbf{W} + \frac{\sigma_0}{\sqrt{n}}\mathbf{U}\right). \nonumber 
\end{IEEEeqnarray}
It is easy to see that $\mathcal{D}_X\mathbf{W} + \frac{\sigma_0}{\sqrt{n}}\mathbf{U}$ is a $n \times s$ Gaussian random matrix. Its $i,j$-th entry can be written as $\hat{\sigma}_i\mathbf{W}_{ij} + \frac{\sigma_0}{\sqrt{n}}\mathbf{U}_{ij}$. It is a Gaussian random variable with mean $0$ and variance $\lambda_i+\frac{\sigma_0^2}{n}$. Hence, we can write \[\frac{1}{\sqrt{n}}(\mathbf{Z}+\pmb{\Xi}) = \frac{1}{\sqrt{s}}\hat{\mathcal{V}}_X\left(\hat{D} + \frac{\sigma_0^2}{n} I \right)^{\frac{1}{2}}\mathbf{W}.\]

Finally, we have
\begin{IEEEeqnarray}{rCl} 
\frac{s}{n}(\mathbf{Z}+\pmb{\Xi})^T(\mathbf{Z}+\pmb{\Xi}) = \mathbf{W}^T\left(\hat{D}+\frac{\sigma_0^2}{n}I\right)\mathbf{W}=\sum_{i=1}^n\left(\hat{\lambda}_i + \frac{\sigma_0^2}{n}\right)\mathbf{w}_i\mathbf{w}_i^T = \sum_{i=1}^n \hat{\lambda}_i^{\xi}\mathbf{w}_i\mathbf{w}_i^T := A^{\xi}. \nonumber
\end{IEEEeqnarray}
Note that $A^{\xi}$ has at most $n$ positive eigenvalues since we are in the overparametrized regime. Now we apply Lemma \ref{lma:eigen_A_bd_beta} to $A^{\xi}$ with $\sigma_w^2 = 1$, we have with probability greater than $1-2e^{-n/c_1}$,
\begin{eqnarray}
\mu_{n}(A^{\xi}) \geq  \mu_{n}(A^{\xi}_k)  &\geq &  (1- \frac{1}{c_1c_2})\sum_{i>k}^n \hat{\lambda}_i^{\xi} - \frac{c_3}{\sqrt{c_1}}\hat{\lambda}_{k+1}^{\xi}n, \nonumber\\
& = & \hat{\lambda}_{k+1}^{\xi} \left((1- \frac{1}{c_1c_2})\sum_{i>k}^n \frac{\hat{\lambda}_{i}^{\xi}}{\hat{\lambda}_{k+1}^{\xi}} - \frac{c_3}{\sqrt{c_1}}n\right), \nonumber \\
& \geq & \hat{\lambda}_{k+1}\left((1- \frac{1}{c_1c_2})\sum_{i>k}^n \frac{\hat{\lambda}_{i}^{\xi}}{\hat{\lambda}_{k+1}^{\xi}} - \frac{c_3}{\sqrt{c_1}}n\right).\label{A_xi_mun}
\end{eqnarray}
Since we assume that $\sum_{i>k}^n \hat{\lambda}_{i}^{\xi}/\hat{\lambda}_{k+1}^{\xi} = \sum_{i>k}^n (\hat{\lambda}_i + \sigma_0^2/n)/(\hat{\lambda}_{k+1}+\sigma_0^2/n) \geq \frac{1}{a}n$. If we adjust $c_1 >1 $ such that $\left((1- \frac{1}{c_1c_2}))an - \frac{c_3}{\sqrt{c_1}}n\right) \geq \frac{1}{c_4}n$, we then have that $\mu_{n}(A^{\xi}) \geq \frac{1}{c_5}n$. %In addition, $\sigma_0^2 = s^{-\alpha}, \alpha \geq 0$, we simultaneously have $\mu_{n}(A^{\xi}) \geq \frac{1}{c_6}n\sigma_0$.

Now for Eq.(\ref{var_z_z}), with probability greater than $1- e^{-n/c_1}$,
\begin{eqnarray}
\textnormal{Eq}.(\ref{var_z_z}) &= & \frac{s\sigma^2}{n} \textnormal{Tr} \left\{\mathbb{E}_x\left[\mathbf{z}_x(\mathbf{W}) \mathbf{z}_x(\mathbf{W})^T\right] \left(\frac{s}{n}\mathbf{Z}_{\xi}^T\mathbf{Z}_{\xi}\right)^{\dagger} \right\} =\frac{\sigma^2}{n} \textnormal{Tr} \left\{\mathbf{W}^T\Sigma\mathbf{W} \left(\frac{s}{n}\mathbf{Z}_{\xi}^T\mathbf{Z}_{\xi}\right)^{\dagger} \right\}, \nonumber \\
&=& \frac{\sigma^2}{n} \textnormal{Tr} \left\{\mathbf{U}^TD\mathbf{U} \left(\frac{s}{n}\mathbf{Z}_{\xi}^T\mathbf{Z}_{\xi}\right)^{\dagger} \right\} =  \frac{\sigma^2}{n} \sum_{i}^p \lambda_i \mathbf{u}_i^T\left(\frac{s}{n}\mathbf{Z}_{\xi}^T\mathbf{Z}_{\xi}\right)^{\dagger} \mathbf{u}_i, \nonumber\\
& \leq &  \frac{\sigma^2}{n}\mu_{n}\left(\frac{s}{n}\mathbf{Z}_{\xi}^T\mathbf{Z}_{\xi}\right)^{-1}\sum_{i}^p \lambda_i \mathbf{u}_i^T\mathbf{u}_i,  \label{trace_normal_normal}\\
&\leq & \frac{\sigma^2}{n} \left\{\frac{1}{c_5}n
\right\}^{-1}c_6 s \sum_{i=1}^p\lambda_i =c_7 \sigma^2 \textnormal{Tr}(\Sigma) \frac{s}{n^2}. \nonumber
\end{eqnarray}
Note for Eq.(\ref{trace_normal_normal}), $\sum_{i=1}^p \lambda_i \mathbf{u}_i^T\mathbf{u}_i$ is a weighted sum of $\chi_1^2$ random variables, with weights given by the $\lambda_i$ in block size of $s$. Hence, if we let $t < s/c_8$, Lemma \ref{lma: sub_exp_sum} gives that with probability $1-2e^{-t}$,
\begin{IEEEeqnarray}{rCl} 
&&\sum_{i=1}^p \lambda_i \mathbf{u}_i^T\mathbf{u}_i \leq  s\sum_{i=1}^p \lambda_i + b \max\left\{\lambda_1 t, \sqrt{st\sum_{i=1}^p \lambda_i^2}\right\}, \nonumber\\
& \leq & s\sum_{i=1}^p \lambda_i + b \max\left\{t\sum_{i=1}^p\lambda_i , \sqrt{st}\sum_{i=1}^p \lambda_i\right\} \leq c_6 s\sum_{i=1}^p \lambda_i \nonumber.
\end{IEEEeqnarray}

Also, with probability greater than $1- e^{-n/c_{9}}$
\begin{IEEEeqnarray}{rCl} 
\textnormal{Eq}.(\ref{var_z_xi}) &= & \frac{s\sigma^2}{n} \textnormal{Tr} \left\{\mathbb{E}_x\left[\mathbf{z}_x(\mathbf{W}) \pmb{\xi}^T\right] \left(\frac{s}{n}\mathbf{Z}_{\xi}^T\mathbf{Z}_{\xi}\right)^{\dagger} \right\} =\frac{s\sigma^2}{n} \mathbb{E}_x \left\{ \pmb{\xi}^T \left(\frac{s}{n}\mathbf{Z}_{\xi}^T\mathbf{Z}_{\xi}\right)^{\dagger}\mathbf{z}_x(\mathbf{W}) \right\}, \nonumber \\
& \leq & \frac{s\sigma^2}{n} \left\|\mathbb{E}_x \left\{ \pmb{\xi}^T \left(\frac{s}{n}\mathbf{Z}_{\xi}^T\mathbf{Z}_{\xi}\right)^{\dagger}\mathbf{z}_x(\mathbf{W}) \right\} \right\| \leq \frac{s\sigma^2}{n} \mu_{n}\left(\frac{s}{n}\mathbf{Z}_{\xi}^T\mathbf{Z}_{\xi}\right)^{-1} \left\|\mathbb{E}_x( \pmb{\xi}^T\mathbf{z}_x(\mathbf{W}))\right\|, \nonumber\\
& \leq & \frac{s\sigma^2}{n} \left\{\frac{1}{c_5}n\right\}^{-1}c_{10}\sigma_0 ~~~(\textnormal{by~Eq.(\ref{bias_z_xi_up})}), \nonumber\\
& =& c_{11}\sigma^2 \frac{s}{n^2}\sigma_0. \nonumber
\end{IEEEeqnarray}
Finally, with probability greater than $1- e^{-N/c_{12}}$
\begin{IEEEeqnarray}{rCl} 
\textnormal{Eq}.(\ref{var_xi_xi}) &= &  \frac{s\sigma^2}{n} \textnormal{Tr} \left\{\pmb{\xi}\pmb{\xi}^T \left(\frac{s}{n}\mathbf{Z}_{\xi}^T\mathbf{Z}_{\xi}\right)^{\dagger} \right\} =\frac{s\sigma^2}{n} \left\{\pmb{\xi}^T \left(\frac{s}{n}\mathbf{Z}_{\xi}^T\mathbf{Z}_{\xi}\right)^{\dagger}\pmb{\xi} \right\} \nonumber \\
& \leq & \frac{s\sigma^2}{n} \pmb{\xi}^T\pmb{\xi} \left\{\frac{1}{c_5}n\right\}^{-1} \leq  \frac{s\sigma^2}{n} \frac{\sigma_0^2}{s}\mathbf{u}^T\mathbf{u} \left\{\frac{1}{c_6}n\right\}^{-1} ~~~(\textnormal{by~Lemma~\ref{lma:norm_sug}})\nonumber\\
& =& c_{13}\sigma^2\frac{s\sigma_0^2}{n^2}. \nonumber
\end{IEEEeqnarray}
Combining above all together, if we choose $c = \max\{c_1, \dots\}$, then with probability greater than $1- 5e^{-n/c}$ \[ \mathbf{V}_{\xi} \leq c\sigma^2\textnormal{Tr}(\Sigma)\frac{s}{n^2}\sigma_0^{-1}.\]
Note that we have omitted the $ \frac{s\sigma_0}{n^2}, \frac{s\sigma_0^2}{n^2}$ terms because they are strictly dominated by the $\frac{s}{n^2}$ term.

\section{Discussion}
The benign overfitting phenomenon has attracted much research interest since it was first observed by \cite{zhang2016understanding,belkin2018understand,belkin2019reconciling}. Our paper continues the line of work in \cite{belkin2019two,bartlett2020benign,hastie2019surprises}, and focuses on developing a theoretical understanding of this phenomenon. Through analyzing the learning risk of the MNLS estimator, we first provide a nearly matching upper and lower bound for the excess learning risk and point out one possible explanation for benign overfitting: the noises $\xi$ in the covariates or the features. While being overlooked in the literature, we discover that $\xi$ plays an important implicit regularization role during learning. Later, by incorporating $\xi$ into our analysis, we explicitly derive how the learning risk is affected by $\xi$. Our analysis describes how the double descent curve happens and in addition, indicates that it is possible to achieve the global optimum bias-variance trade-off by varying the decay rate of $\xi$. Our results may shed new light on the theoretical understandings of modern deep learning, which open doors for future studies of the design of the deep learning architecture. Furthermore, our results apply to any finite sample data size or asymptotic case with arbitrary data dimension and rely on very weak assumptions of the kernel with almost no assumptions on the data generating distribution.

There are several extensions that we believe worth exploring. Firstly, although our results shed light on the two-layer neural network with fixed first layer weights, we would like to understand what would happen if we could optimize the first layer weights, i.e. optimizing $\mathbf{W}$ in our model. In literature, there are many tools in analyzing neural network optimization with connection to its generalization error. Examples include the transport map formulation \cite{suzuki2020generalization}, mean-field analysis \cite{mei2018mean,sirignano2020mean}, and neural tangent kernel \cite{du2019gradient,jacot2018neural}. Therefore, how to utilize these tools to investigate the effect of the noise $\xi$ during neural network training would be an interesting direction. Another direction is to analyze the role of the noise $\xi$ in the models with different loss functions.

\paragraph{Acknowledgment} The authors would like to thank Chao Zhang and Zhongyi Hu for fruitful discussion and proofreading.

\bibliographystyle{unsrt}
\bibliography{ref}

\newpage

\appendix
\section{Bias-Variance Tradeoff: Realizable Case} \label{appen:bias-var}
\begin{proof}
Recall the MNLS estimator has the form of \[\tilde{f}(x) = \mathbf{z}_{x}(\mathbf{w})^T\tilde{\beta} = \mathbf{z}_{x}(\mathbf{w})^T(\mathbf{Z}^T\mathbf{Z})^{\dagger}\mathbf{Z}^TY.\]
Under our assumptions, $Y = f_{*}(X) + \pmb{\epsilon}$, where $f_*(X) = \mathbf{Z}\beta_*$ and $\pmb{\epsilon} = [\epsilon_1,\cdots,\epsilon_n]^T$. Hence,
\begin{IEEEeqnarray}{rCl}
R(\tilde{\beta}) &= &\mathbb{E}_{x,\epsilon}(\tilde{f}(x)- f_*(x))^2 = \mathbb{E}_{x,\epsilon} \left(\mathbf{z}_{x}(\mathbf{w})^T\tilde{\beta} - \mathbf{z}_{x}(\mathbf{w})^T\beta_*\right)^2, \nonumber \\
& =& \mathbb{E}_{x,\epsilon} \left\{\mathbf{z}_{x}\left(\mathbf{w})^T(\mathbf{Z}^T\mathbf{Z}\right)^{\dagger}\mathbf{Z}^T( f_{*}(X) + \pmb{\epsilon}) - \mathbf{z}_{x}(\mathbf{w})^T\beta_*\right\}^2, \nonumber \\
& = & \mathbb{E}_{x,\epsilon} \left\{\mathbf{z}_{x}(\mathbf{w})^T\left(\mathbf{Z}^T\mathbf{Z}\right)^{\dagger}\mathbf{Z}^T \pmb{\epsilon} + \mathbf{z}_{x}(\mathbf{w})^T\left((\mathbf{Z}^T\mathbf{Z})^{\dagger}\mathbf{Z}^T\mathbf{Z}-I\right)\beta_*\right\}^2, \nonumber \\
& = & \int_{\mathcal{X}} \left\|\mathbf{z}_x(\mathbf{W})^T\left[(\mathbf{Z}^T\mathbf{Z})^{\dagger}\mathbf{Z}^T\mathbf{Z} - I\right]\beta_*\right\|^2d\rho(x) + \int_{\mathcal{X}} \mathbb{E}_{\epsilon}\left\| \mathbf{z}_x(\mathbf{W})^T(\mathbf{Z}^T\mathbf{Z})^{\dagger}\mathbf{Z}^T\pmb{\epsilon}\right\|^2 d\rho(x).\nonumber
\end{IEEEeqnarray}
We then replace $\pmb{\epsilon} $ with $Y -f_*(X)$.
\end{proof}

\section{Bias-Variance Tradeoff: Unrealizable Case} \label{appen:bias-var-unrea}
\begin{proof}
We decompose the excess risk as follows:
\begin{IEEEeqnarray}{rCl}
R(\tilde{\beta}) &= &\mathbb{E}_{x,\epsilon}\left(\tilde{f}(x)- f_*(x)\right)^2 = \mathbb{E}_{x,\epsilon} \left(\mathbf{z}_{x}(\mathbf{w})^T\tilde{\beta} - f_*(x) \right)^2, \nonumber \\
& =& \mathbb{E}_{x,\epsilon} \left\{\mathbf{z}_{x}(\mathbf{w})^T(\mathbf{Z}^T\mathbf{Z})^{\dagger}\mathbf{Z}^TY - f_{\tilde{\mathcal{H}}}(x) + \left[f_{\tilde{\mathcal{H}}}(x)-f_*(x)\right]\right\}^2, \nonumber \\
& =& \mathbb{E}_{x,\epsilon} \left\{\mathbf{z}_{x}(\mathbf{w})^T(\mathbf{Z}^T\mathbf{Z})^{\dagger}\mathbf{Z}^T\left(f_*(X)- f_{\tilde{\mathcal{H}}}(X) + f_{\tilde{\mathcal{H}}}(X) + \pmb{\epsilon}\right) - f_{\tilde{\mathcal{H}}}(x) + \left[f_{\tilde{\mathcal{H}}}(x)-f_*(x)\right]\right\}^2, \nonumber \\
& \leq & 3 \mathbb{E}_{x,\epsilon} \left\{\mathbf{z}_{x}(\mathbf{w})^T(\mathbf{Z}^T\mathbf{Z})^{\dagger}\mathbf{Z}^T \left(f_{\tilde{\mathcal{H}}}(X) + \pmb{\epsilon} \right)- f_{\tilde{\mathcal{H}}}(x)  \right\}^2 := A \nonumber \\
&& + 3\left\{\mathbb{E}_{x} \left\{\mathbf{z}_{x}(\mathbf{w})^T(\mathbf{Z}^T\mathbf{Z})^{\dagger}\mathbf{Z}^T \left(f_*(X) - f_{\tilde{\mathcal{H}}}(X) \right) \right\}^2 + \mathbb{E}_{x} \left(f_*(x) - f_{\tilde{\mathcal{H}}}(x)\right)^2 \right\}:= \mathbf{M}_R.\nonumber
\end{IEEEeqnarray}
Now we can see that the risk has been decomposed into the $A$ term and the misspecification error term $\mathbf{M}_R$. For the $A$ term, it is similar to the excess risk defined in the realizable case, hence can be decomposed as $\mathbf{B}_R$ and $\mathbf{V}_R$.
\end{proof}

\section{Bias-related Inequaity}
\begin{lma} \label{lma:proj}
Let $\mathbf{Z} \in \mathbb{R}^{n\times s}$ be a feature matrix. Recall $\Pi = (\mathbf{Z}^T\mathbf{Z})^{\dagger}\mathbf{Z}^T\mathbf{Z} - I $, we have $\|\Pi\| = 0$, if $s \leq n$;  Otherwise, $\|\Pi\| \leq 1$. 
\end{lma}

\begin{proof}
In case $s <n$,  we have $(\mathbf{Z}^T\mathbf{Z})^{\dagger} = (\mathbf{Z}^T\mathbf{Z})^{-1}$. As a result $\|\Pi\| = 0$. If $s> n$, by \cite[Lemma 1]{hastie2019surprises}, we know that $I -(\mathbf{Z}^T\mathbf{Z})^{\dagger}\mathbf{Z}^T\mathbf{Z}$ is a projection onto the null space of $\mathbf{Z}$. Hence, if we pick any vector from that space, we have $(I -(\mathbf{Z}^T\mathbf{Z})^{\dagger}\mathbf{Z}^T\mathbf{Z})v = v$. Thus,\[\|I -(\mathbf{Z}^T\mathbf{Z})^{\dagger}\mathbf{Z}^T\mathbf{Z}\| \leq 1.\]
\end{proof}

\section{Probability Bound on Sequences and Matrix Norms}
\begin{lma} \label{lma:sug_suexp}
Let $x$, $y$ be centered $\sigma_x^2$-subgaussian and $\sigma_y^2$-subgaussian random variables respectively, i.e.,\[\mathbb{E}(\exp(tx)) \leq \exp(\frac{\sigma_x^2}{2}t^2),~~~~~\mathbb{E}(\exp(ty)) \leq \exp(\frac{\sigma_y^2}{2}t^2).\] We then have that, for a universal constant $b$, the product of $x$ and $y$ is a centered $b \sigma_x \sigma_y$-subexponential random variable, i.e. \[\mathbb{E}(\exp(txy)) \leq \exp(b^2 \sigma_x^2\sigma_y^2 t^2).\]
\end{lma}
\begin{proof}
We compute the moment generating function directly,
\begin{IEEEeqnarray}{rCl}
\mathbb{E}(\exp(txy)) &=& \mathbb{E}\left\{\mathbb{E}\left(\exp(txy)\bigg| y\right)\right\}, \nonumber\\
& \leq & \mathbb{E}\left\{\exp\left(\frac{\sigma_x^2}{2}t^2y^2\right)\right\}, \nonumber \\
& \leq & \exp\left(c_1 \frac{\sigma_x^2}{2}\frac{\sigma_y^2}{2}t^2\right), \nonumber \\
& =& \exp(c \sigma_x^2\sigma_y^2 t^2), \nonumber
\end{IEEEeqnarray}
provided $|t| \leq (c\sigma_x\sigma_y)^{-1}$. Note that for the last inequality, we have used the results from \cite[Proposition 2.5.2]{vershynin2018high}.
\end{proof}

Below are two concentration results for subexponential and subgaussian random variables from \cite[Corollary S.6 \& S.7]{bartlett2020benign}.

\begin{lma}\label{lma: sub_exp_sum}
Suppose we have a sequence of non-increasing and non-negative numbers $\{\lambda_i\}_{i=1}^{\infty}$ such that $\sum_{i=1}^{\infty} \lambda_i < \infty$. In addition, we have a sequence of i.i.d centered, $\sigma$-subexponential random variables $\{\xi_i\}_{i=1}^{\infty}$, then there is a universal constant $b$ such that for probability greater than $1-2e^{-t}$, $t >0$, we have
\[\lvert\sum_{i} \lambda_i\xi_i \rvert \leq b\sigma \max\left(\lambda_1 t, \sqrt{t\sum_i\lambda_i^2} \right). \]
\end{lma}

\begin{lma}\label{lma:norm_sug}
let $\mathbf{w} \in \mathbb{R}^n$ be a random vector with each coordinate being a mean $0$, unit variance, $\sigma_w^2$-subgaussian random variable. Then there is a universal constant $b$ such that with probability greater than $1- 2e^{-t}$, we have \[\|\mathbf{w}\|^2 \leq  n + b\sigma^2_w \left(t+ \sqrt{nt}\right).\] In particular, if $t <\frac{n}{b_1}$, \[\|\mathbf{w}\|^2 \leq  b_2 n, \] for some universal constants $b_1,b_2$.
\end{lma}

Lemma \ref{lma:e-net} are from \cite[Lemma S.8]{bartlett2020benign}.
\begin{lma} \label{lma:e-net} $(\epsilon \text{-net argument}) $  
Let $A \in \mathbb{R}^{n \times n}$ be a symmetric matrix, and $\mathcal{N}_{\epsilon}$ is an $\epsilon$-net on the unit sphere $\mathcal{S}^{n-1}$ with $\epsilon<\frac{1}{2}$. Then we have \[\|A\| \leq(1-\epsilon)^{-2} \max _{\mathbf{a} \in \mathcal{N}_{\epsilon}}\left|\mathbf{a}^{T} A \mathbf{a}\right|.\]
\end{lma}

\begin{lma} \label{lma:con_norm_pro}
Let $\{w_{i}\}_{i=1}^{\infty}, \{u_i\}_{i=1}^{\infty}$ be two sequences of i.i.d random variables distributed as standard normal. Moreover, let $\{\lambda_i\}_{i=1}^{\infty}$ be a sequence of non-negative, non-increasing numbers such that $\sum_{i=1}^\infty \lambda_i^2 < \infty$. Then with probability greater than $1-2e^{-t}$, we have 
\[-\left(b_1 \lambda_1t + \sqrt{b_2t\sum_i \lambda_i^2}\right) \leq \sum_{i=1}^{\infty} \lambda_i w_i u_i \leq b_1 \lambda_1t + \sqrt{b_2t\sum_i \lambda_i^2}. \] 
\end{lma}

\begin{proof}
Using Markov Inequality and for any $\tau > 0$, we have that:
\begin{IEEEeqnarray}{rCl}
P(\sum_{i=1}^{\infty} \lambda_i w_i u_i \geq t) & = & P(\exp(\tau\sum_{i=1}^{\infty} \lambda_i w_i u_i) \geq e^{\tau t})\nonumber\\
&\leq & e^{-\tau t} \prod_{i=1} \mathbb{E}(e^{\tau \lambda_i w_iu_i}) \nonumber\\
& = &  e^{-\tau t} \prod_{i=1} (\frac{1}{1 - (\tau \lambda_i)^2})^{\frac{1}{2}} ~~~~(\textnormal{by~ Lemma~\ref{lma:mgf}})\nonumber\\
& =& \exp(-\tau t - \frac{1}{2} \sum_{i=1}\log(1 - (\tau\lambda_i)^2)), \nonumber
\end{IEEEeqnarray}
provided $\tau\lambda_i < 1$. Now for any $x \in (0,1)$,
\begin{IEEEeqnarray}{rCl}
\log(1-x) &= & - \int_{0}^{x} \frac{1}{1-t} dt\nonumber\\
& \geq & -\int_{0}^x \frac{1}{(1-t)^2} dt\nonumber\\
& = & -\frac{x}{1-x}. \nonumber
\end{IEEEeqnarray}
As a result
\begin{IEEEeqnarray}{rCl}
\exp(-\tau t - \frac{1}{2}\sum_{i=1}\log(1 - (\tau\lambda_i)^2)) &\leq & \exp(-\tau t - \sum_{i=1} -\frac{(\tau \lambda_i)^2}{1- (\tau\lambda_i)^2}) \nonumber \\
&\leq  & \exp(-\tau t + \frac{\tau^2}{1-(\tau\lambda_1)^2} \sum_{i=1}\lambda_i^2) \nonumber\\
& \leq & \exp\left\{-\tau t + \frac{\tau^2}{1-(\tau\lambda_1)^2} \sum_{i=1}\lambda_i^2 + \frac{\tau}{2\lambda_1}[\log(\frac{1}{\tau}+\lambda_1)-\log(\frac{1}{\tau}-\lambda_1)]\right\} \nonumber
\end{IEEEeqnarray}
Note that for the last inequality, we have $\frac{1}{\tau}-\lambda_1 > 0$, since $\tau \lambda_1 < 1$.
Now we let \[x = \tau t - \frac{\tau^2}{1-(\tau\lambda_1)^2} \sum_{i=1}\lambda_i^2 - \frac{\tau}{2\lambda_1}[\log(\frac{1}{\tau}+\lambda_1)-\log(\frac{1}{\tau}-\lambda_1)],\] we have \[t = \frac{x}{\tau} - \frac{1}{\tau}\frac{1}{(\frac{1}{\tau})^2-\lambda_1} \sum_{i=1} \lambda_i^2 - \frac{1}{2\lambda_1}[\log(\frac{1}{\tau}+\lambda_1)-\log(\frac{1}{\tau}-\lambda_1)].\] Optimize over $\tau^{-1}$, this gives \[\tau^{-1} = \sqrt{\lambda_1^2 + \sqrt{\lambda_1^2\frac{\sum_i\lambda_i^2}{x}}}.\] Checking the conditions, we see that $\tau\lambda_i = \lambda_i \left(\lambda_1^2 + \sqrt{\lambda_1^2\frac{\sum_i\lambda_i^2}{x}}\right)^{-\frac{1}{2}} < 1$, implying $\tau^{-1}$ is a valid choice. Now with this choice, some simple algebra shows that there are constants $c_1,c_2$ such that \[t \leq c_1 \lambda_1x + \sqrt{c_2x\sum_i \lambda_i^2}.\]  
This gives \[P\left(\sum_{i=1}^{\infty} \lambda_i w_i u_i \geq c_1 \lambda_1x + \sqrt{c_2x\sum_i \lambda_i^2}\right) \leq \exp(-x).\]

For the lower bound we use the symmetry of Gaussian distribution as:
\begin{IEEEeqnarray}{rCl}
P\left(\sum_{i=1}^{\infty} \lambda_i w_i u_i \leq -t \right) &=& P\left(-\sum_{i=1}^{\infty} \lambda_i w_i u_i \geq t \right) \nonumber\\
& =& P\left(\sum_{i=1}^{\infty} \lambda_i (-w_i) u_i \geq t \right) \nonumber \\
& = & P\left(\sum_{i=1}^{\infty} \lambda_i w_i u_i \geq t \right).\nonumber
\end{IEEEeqnarray}
We complete the proof by using the union bound for the probability for both the upper and lower bound.
\end{proof}

\begin{lma}\label{lma:eigen_norm_pro}
Let $\{\mathbf{w}_i\}_{i=1}^{\infty}$ and $\{\mathbf{u}_i\}_{i=1}^{\infty}$ be two sequence of random vectors where the entry of each $\mathbf{w}_i \in \mathbb{R}^n$ and $\mathbf{u}_i \in \mathbb{R}^n$ is i.i.d $\mathcal{N}(0,1)$. Furthermore, let $\{\lambda_i\}_{i=1}^{\infty}$ be a non-negative sequence such that $\{\lambda_i^2\}_{i=1}^{\infty}$ is non-increasing and $\sum_{i=1} \lambda_i^2 < \infty$. Denote $A = \sum_{i=1} \lambda_i \mathbf{w}_i\mathbf{u}_i^T$. Let $\mu_{1}(A) \geq ,\cdots, \geq \mu_{n}(A)$ be its eigenvalues. Then with probability at least $1 - 2e^{-n/b_1}$, we have 
\begin{IEEEeqnarray}{rCl}
-b_2n\leq  \mu_{n}(A) \leq \mu_1(A) \leq b_2n \nonumber.
\end{IEEEeqnarray}
\end{lma}
\begin{proof}
For any unit vector $\mathbf{v} \in \mathbb{R}^n$, we have that both $\mathbf{v}^T\mathbf{w}_i$ and $\mathbf{v}^T\mathbf{u}_i$ are independent and distributed as $\mathcal{N}(0,1)$. Using Lemma \ref{lma:con_norm_pro}, we have with $1-2e^{-t}$,
\begin{IEEEeqnarray}{rCl}
|\mathbf{v}^T A \mathbf{v}| &=& |\sum_{i=1}\lambda_i \mathbf{v}^T\mathbf{w}_i\mathbf{u}_i^T\mathbf{v}| \nonumber\\
& \leq & c_1 \lambda_1t + \sqrt{c_2t\sum_i \lambda_i^2} \nonumber
\end{IEEEeqnarray}
Now we apply the $\epsilon$-net method to $\mathcal{S}^{n-1}$ with $\epsilon = \frac{1}{4}$, implying $|\mathcal{N}_{\epsilon}| < 9^n$. This gives \[\|A\| \leq c_3\lambda_1(t+ n\log 9) + c_4  \sqrt{(t+ n \log9)\sum_i\lambda_i^2}.\]
Now when $t \leq \frac{n}{c_5}$. Hence, we have \[\|A\| \leq c_6 \lambda_1 n + c_7 \sqrt{n \sum_{i=1}\lambda_i^2}.\]
Since $\sum_{i=1}\lambda_i^2 < \infty$, we can further simplify that $\|A\| \leq c_8 n$.
\end{proof}

\begin{lma}\label{lma:eigen_norm_sug}
Let $\{\mathbf{w}_i\}_{i=1}^{\infty}$ and $\{\mathbf{u}_i\}_{i=1}^{\infty}$ be two sequence of random vectors where the entry of each $\mathbf{w}_i \in \mathbb{R}^n$ and $\mathbf{u}_i \in \mathbb{R}^n$ are i.i.d centered, unit variance $\sigma_w^2$ and $\sigma_u^2$-subgaussian random variables. Furthermore, let $\{\lambda_i\}_{i=1}^{\infty}$ be a sequence of non-negative, non-increasing numbers such that $\sum_{i=1} \lambda_i^2 < \infty$. Denote $A = \sum_{i=1} \lambda_i \mathbf{w}_i\mathbf{u}_i^T$. Let $\mu_{1}(A) \geq ,\cdots, \geq \mu_{n}(A)$ be its eigenvalues. Then with probability at least $1 -e^{-n/b_1}$, we have 
\begin{IEEEeqnarray}{rCl}
-b_{2}\sigma_w\sigma_u n \leq  \mu_{n}(A) \leq \mu_1(A) \leq b_{2}\sigma_w\sigma_u n\nonumber.
\end{IEEEeqnarray}
\end{lma}

\begin{proof}
The proof is similar to Lemma \ref{lma:eigen_norm_pro} by noticing that for any unit vector $\mathbf{v} \in \mathbb{R}^n$, $\mathbf{v}^T\mathbf{w}_i\mathbf{u}_i^T\mathbf{v}$ is a $c_{1}\sigma_{w}\sigma_u$-subexponential random variable according to Lemma \ref{lma:sug_suexp}. Applying Lemma \ref{lma: sub_exp_sum} and $\epsilon$-net argument to every unit vector in $\mathcal{S}^{n-1}$ yields our results.
\end{proof}

\section{Moment Generating Function of Product of Random Variables }
\begin{lma}\label{lma:mgf}
Let $x,y$ be two independent random variables from $\mathcal{N}(0,1)$, then the moment generating function of the product is as following:
\begin{IEEEeqnarray}{rCl}
M_{xy}(t) = \frac{1}{\sqrt{1-t^2}}, \nonumber
\end{IEEEeqnarray}
provided $t < 1$.
\end{lma}

\begin{proof}
By definition, assuming $t < 1$, we have
\begin{IEEEeqnarray}{rCl}
\mathbb{E}(e^{txy}) &= & \frac{1}{2\pi}\int_{-\infty}^{\infty} \int_{-\infty}^{\infty} e^{txy} e^{-\frac{x^2}{2}}e^{-\frac{y^2}{2}}dxdy, \nonumber \\
& =& \frac{1}{2\pi}\int_{-\infty}^{\infty} e^{-\frac{y^2}{2}} \int_{-\infty}^{\infty} \exp\left(-\frac{1}{2}(x^2-2txy)\right)dxdy, \nonumber\\
& =& \frac{1}{2\pi}\int_{-\infty}^{\infty} e^{-\frac{y^2}{2}} e^{\frac{t^2}{2}y^2}\int_{-\infty}^{\infty} \exp\left(-\frac{1}{2}(x-ty)^2\right)dxdy, \nonumber \\
& =&  \frac{1}{\sqrt{2\pi}}\int_{-\infty}^{\infty}\exp\left(-\frac{1}{2}(1-t^2)y^2\right) dy, \nonumber\\
& =& \frac{1}{\sqrt{1-t^2}}. \nonumber
\end{IEEEeqnarray}
\end{proof}

\section{Proof of Corollary \ref{theo: over_rff_noisy_sug}}
The proof is similar to the proof of Theorem \ref{theo: over_rff_noisy} except we will use a different concentration inequality. As usual, we start with bias-variance decomposition. In the subgaussian scenario, the bias-viarance decomposition is exactly the same as that in the Gaussian case, with the only difference being the shape of the noise. As a result, we will adopt the same notation as in Section \ref{sec:gau_proof} except that the noise is now $\sigma_0^2\sigma_u^2$-subgaussian.
\subsection{Upper Bound on Bias}
\begin{IEEEeqnarray}{rCl} 
\mathbf{B}_{\xi} &= &  \beta_*^{\xi T}  \Pi_{\xi}\left\{ \mathbb{E}_x\left(\mathbf{z}_x^{\xi}(\mathbf{W})\mathbf{z}_x^{\xi}(\mathbf{W})^T \right)- \frac{1}{n}\mathbf{Z}_{\xi}^T\mathbf{Z}_{\xi}\right\}\Pi_{\xi}\beta_*^{\xi}, \nonumber
\end{IEEEeqnarray} 
where 
\begin{IEEEeqnarray}{rCl} 
\mathbb{E}_x\left(\mathbf{z}_x^{\xi}(\mathbf{W})\mathbf{z}_x^{\xi}(\mathbf{W})^T \right)- \frac{1}{n}\mathbf{Z}_{\xi}^T\mathbf{Z}_{\xi} &=& \mathbb{E}_{x}\left\{ \mathbf{z}_x(\mathbf{W})\mathbf{z}_x(\mathbf{W})^T -\frac{1}{n} \mathbf{Z}^T\mathbf{Z}\right\}, \nonumber\\
&&+ \mathbb{E}_{x} (\mathbf{z}_x(\mathbf{W})\pmb{\xi}^T ) + \mathbb{E}_{x} (\pmb{\xi}\mathbf{z}_x(\mathbf{W})^T )+  \pmb{\xi}\pmb{\xi}^T\nonumber, \\
&& - \frac{1}{n}\mathbf{Z}^T\Xi  -\frac{1}{n}\Xi^T\mathbf{Z}  - \frac{1}{n}\Xi^T\Xi. \nonumber
\end{IEEEeqnarray}

We need to upper bound $\mathbb{E}_{x} (\mathbf{z}_x(\mathbf{W})\pmb{\xi}^T )$, $\pmb{\xi}\pmb{\xi}^T$, $\frac{1}{n}\mathbf{Z}^T\Xi$ and $\frac{1}{n}\Xi^T\Xi$.\\

Firstly, for $\mathbb{E}_{x} (\mathbf{z}_x(\mathbf{W})\pmb{\xi}^T )$, we have  \[\mathbf{z}_x(\mathbf{W})\pmb{\xi}^T = \frac{\sigma_0}{s} \mathbf{W}^TD^{\frac{1}{2}}V_x \mathbf{u}^T = \frac{\sigma_0}{s} \sqrt{k(x,x)} \mathbf{v} \mathbf{u}^T. \]

where we recall $\mathbf{v} \in \mathbb{R}^s$ is a vector with each entry being i.i.d $\sim \mathcal{N}(0,1)$ and $\mathbf{u} \in \mathbb{R}^s$ is a vector with each entry being i.i.d $\sigma_u^2$-subgaussian. We can upper bound $\mathbb{E}_x(\mathbf{z}_x(\mathbf{W})\pmb{\xi}^T)$ with probability greater than $1- e^{-s/c_1}$ as:
\begin{IEEEeqnarray}{rCl} 
\left\|\mathbb{E}_x(\mathbf{z}_x(\mathbf{W})\pmb{\xi}^T)\right\| & \leq & \frac{\sigma_0}{s} \left\|\sqrt{\mathbb{E}_x(k(x,x))}\right\| \|\mathbf{w}\mathbf{u}^T\|, \nonumber\\
& =& \frac{\sigma_0}{s} \left\|\sqrt{\mathbb{E}_x(k(x,x))}\right\| \left|\sum_{i=1}^s u_iv_i\right|, \nonumber \\
& \leq &  \frac{\sigma_0}{s} \sqrt{C_0} c_2 s, ~~~\textnormal{By~Lemma~\ref{lma:eigen_norm_sug}~and~let~$n=1$} \nonumber\\
& = & c_3 \sigma_0. \nonumber
\end{IEEEeqnarray}

Secondly, for $\pmb{\xi}\pmb{\xi}^T$, $\|\pmb{\xi}\pmb{\xi}^T\| =\pmb{\xi}^T\pmb{\xi} = \frac{\sigma_0^2}{s} \|\mathbf{u}\|^2 \leq c_4 \sigma_0^2 $ with probability greater than $1- e^{-s/c_5}$ by Lemma \ref{lma:norm_sug}.

Moreover, for $\frac{1}{n}\mathbf{Z}^T\Xi$, using the rotation invariance for subgaussian distribution, we similarly have
\begin{IEEEeqnarray}{rCl} 
\frac{1}{n}\mathbf{Z}^T\Xi =\frac{\sigma_0}{s\sqrt{n}} \sum_i^n \sqrt{\lambda}_i \mathbf{w}_i \mathbf{u}_i^T \nonumber
\end{IEEEeqnarray}
Now each entry of $\mathbf{w}_i$ is $1$-subgaussian and $\mathbf{u}_i$ is $\sigma_u^2$-subgaussian, applying Lemma \ref{lma:eigen_norm_sug}, with probability greater than $1-e^{-s/c_6}$ we obtain
\begin{IEEEeqnarray}{rCl} 
\left\|\frac{1}{n}\Xi^T\mathbf{Z}\right\| \leq \frac{\sigma_0}{s\sqrt{n}}c_{7} \sigma_u n\sqrt{\sum_i \hat{\lambda}_i} = c_8 \sigma_0\frac{n}{s\sqrt{n}}\leq c_9 \frac{\sigma_0}{\sqrt{s}}.\nonumber
\end{IEEEeqnarray}

Finally, for $\frac{1}{n}\Xi^T\Xi$, by appealing to \cite[Lemma 4]{bartlett2020benign}, which is the subgaussian version of Lemma \ref{lma:eigen_A_bd} and asserts that $\|\sum_i^n \mathbf{u}_i \mathbf{u}_i^t\| \leq c_{10} n$ with probability greater than $1-e^{-s/c_{11}}$, we can upper bound
\begin{IEEEeqnarray}{rCl} 
\left\|\frac{1}{n}\Xi^T\Xi \right\| & =& \left\|\frac{\sigma_0^2}{n} \sum_i^n \mathbf{u}_i\mathbf{u}_i^T\right\|, \nonumber\\
& \leq & \frac{\sigma_0^2}{n} \left\|\sum_i^n \mathbf{u}_i\mathbf{u}_i^T\right\|,\nonumber\\
& \leq & c_{12} \frac{\sigma_0^2n}{ns}\leq c_{13} \frac{\sigma_0^2}{s}.\nonumber
\end{IEEEeqnarray}

Combining above results together, with probability greater than $1-\delta- e^{-n/c}$ we upper bound the bias as: \[\mathbf{B}_{\xi} \leq  c\left\{ \frac{\lambda_W}{s} \|\Sigma\|\sqrt{\log(\frac{14r(\Sigma)}{\delta})/n}  + \sigma_0 + \sigma_0^2 \right\} \|\Pi_{\xi}\|^2\|\beta^{\xi}_*\|^2. \]

\subsection{Upper Bound on Variance}
\begin{IEEEeqnarray}{rCl} 
\mathbf{V}_{\xi} & =&  \frac{\sigma^2}{n} \textnormal{Tr} \left\{\mathbb{E}_x\left[\mathbf{z}_x(\mathbf{W}) \mathbf{z}_x(\mathbf{W})^T\right] \left(\frac{1}{n}\mathbf{Z}_{\xi}^T\mathbf{Z}_{\xi}\right)^{\dagger} \right\}, \nonumber \\
&& + \frac{\sigma^2}{n} \textnormal{Tr} \left\{\mathbb{E}_x\left[\mathbf{z}_x(\mathbf{W}) \pmb{\xi}^T\right] \left(\frac{1}{n}\mathbf{Z}_{\xi}^T\mathbf{Z}_{\xi}\right)^{\dagger} \right\}, \nonumber \\
&& + \frac{\sigma^2}{n} \textnormal{Tr} \left\{\mathbb{E}_x\left[\pmb{\xi}\mathbf{z}_x(\mathbf{W})^T\right] \left(\frac{1}{n}\mathbf{Z}_{\xi}^T\mathbf{Z}_{\xi}\right)^{\dagger} \right\}, \nonumber \\
&& + \frac{\sigma^2}{n} \textnormal{Tr} \left\{\pmb{\xi}\pmb{\xi}^T \left(\frac{1}{n}\mathbf{Z}_{\xi}^T\mathbf{Z}_{\xi}\right)^{\dagger} \right\}, \nonumber 
\end{IEEEeqnarray}
Through similar analysis to the Gaussian case, use Eq.(\ref{assum_sub}), we can also have that $\mu_{n}(\frac{1}{n}\mathbf{Z}_{\xi}^T\mathbf{Z}_{\xi}) \geq \frac{1}{a}n$. Following the similar argument as in the Gaussian case we can upper bound the variance with probability greater than $1- e^{-n/c}$ \[\mathbf{V}_{\xi} \leq c\sigma^2\textnormal{Tr}(\Sigma)  \frac{s}{n^2}.\]

\section{Proof of Corollary \ref{theo:double_descent}}
\begin{proof}
Proof of Corollary \ref{theo:double_descent} is simply applying $\sigma_0^2 = O(s^{-\alpha})$ into Eq.(\ref{noi_bia_up_sub}) and Eq.(\ref{noi_var_up_sub}). 
\end{proof}

%%%%%%%%%%%%%%%%%%%%%%%%%%%%%%%%%%%
% In the context where there is a misspecification error, we have three estimators:
% \begin{itemize}
%     \item $f_*$: $y = f_*(x) + \epsilon$;
%     \item $f_H$: best predictor in the space $H$ and $f_H = \mathbf{z}_x(\mathbf{w})^T\beta_H$;
%     \item $\tilde{f}$: the interpolating estimator.
% \end{itemize}

\end{document}